\documentclass[a4paper, 10 pt]{IEEEtran}

\IEEEoverridecommandlockouts

\title{\LARGE \bf A Real-Time Game Theoretic Planner for \\Autonomous Two-Player Drone Racing}

\author{Riccardo Spica, Davide Falanga, Eric Cristofalo, Eduardo Montijano, Davide Scaramuzza, Mac Schwager
\thanks{Toyota Research Institute ("TRI") provided funds to assist the authors with their research, but this article solely reflects the opinions and conclusions of its authors and not TRI or any other Toyota entity. This research was also supported by ONR grant N00014-16-1-2787 and a National Defense Science and Engineering Graduate Fellowship. We are grateful for this support.}
\thanks{This research was supported by the SNSF-ERC Starting Grant and the National Centre of Competence in Research (NCCR) Robotics, through the Swiss National Science Foundation.}
\thanks{This research was supported by the Spanish project Ministerio de Econom\'ia y Competitividad DPI2015-69376-R.}
\thanks{R. Spica, E. Cristofalo and M. Schwager are with the Department of Aeronautics and Astronautics, Stanford University, Stanford, CA 94305, USA {\tt\scriptsize rspica;ecristof;schwager@stanford.edu}}
\thanks{D. Falanga and D. Scaramuzza are with the Robotics and Perception Group, Dep. of Informatics, University of Zurich, and Dep. of Neuroinformatics, University of Zurich and ETH Zurich, Switzerland--{\tt\scriptsize http://rpg.ifi.uzh.ch}.}
\thanks{E. Montijano is with Instituto de Investigaci\'on en Ingenier\`ia de Arag\'on, Universidad de Zaragoza, Zaragoza 50018, Spain {\tt\scriptsize emonti@unizar.es}.}
}

\usepackage[utf8]{inputenc}
\usepackage[T1]{fontenc}
\usepackage[english]{babel}

\usepackage{amsthm}
\usepackage{amsfonts,amssymb,amsmath,mathtools}
\usepackage{blkarray}

\usepackage{dblfloatfix}
\usepackage{graphicx}
\graphicspath{{./images/}}
\usepackage{subcaption}
\usepackage{xcolor}
\usepackage{tikz}
\usetikzlibrary{shapes,arrows,calc,positioning,patterns}
\usepackage{images/aircraftshapes}

\usepackage{hyperref}
\usepackage{cleveref}
\usepackage{cite}

\usepackage{siunitx}
\usepackage{ifthen}
\usepackage{todonotes}

\usepackage{verbatim}

\usepackage{enumitem}

\newcommand{\st}{:}
\newcommand{\defas}{\coloneqq}

\newcommand{\const}[1]{\overline{#1}}

\newcommand{\MixParDeriv}[3][]{\frac{\partial^{#1}#2}{{#3}}}
\newcommand{\parDeriv}[3][]{\frac{\partial^{#1}#2}{\partial {#3}^{#1}}}
\newcommand{\totDeriv}[3][]{\frac{\textrm{d}^{#1}#2}{\textrm{d}{#3}^{#1}}}

\newcommand{\norm}[1]{\left\|#1\right\|}
\newcommand{\abs}[1]{\left|#1\right|}
\newcommand{\evalin}[2]{{\left.#1\right|}_{#2}}
\newcommand{\kron}{\otimes}

\DeclareMathOperator{\atan}{atan2}

\newcommand{\real}[1]{\mathbb{R}^{#1}}

\newcommand{\setS}[1]{\mathbb{S}^{#1}}

\newcommand{\vect}[1]{\boldsymbol{#1}}
\newcommand{\matr}[1]{\boldsymbol{#1}}

\newcommand{\rotMat}[2]{{\prescript{#1}{}{\matr{R}}}_{#2}}

\newcommand{\yaw}{\psi}

\newcommand{\ex}{\vect{e}_1}
\newcommand{\ey}{\vect{e}_2}

\newcommand{\eye}[1]{\matr{I}_{#1}}
\newcommand{\zeros}[2]{
\ifthenelse{\equal{#2}{1}}{\vect{0}_{#1}}{\matr{\cancel{O}}_{#1 \times #2}}
}
\newcommand{\ones}[2]{
\ifthenelse{\equal{#2}{1}}{\vect{1}_{#1}}{\matr{1}_{#1 \kron #2}}
}

\newcommand{\posElem}{p}
\newcommand{\pos}{\vect{\posElem}}
\newcommand{\bearingElem}{\beta}
\newcommand{\bearing}{\vect{\bearingElem}}
\newcommand{\dist}{d}
\newcommand{\distMax}{\overline{\dist}}
\newcommand{\linVelElem}{v}
\newcommand{\linVelMax}{\overline{\linVelElem}}
\newcommand{\linVel}{\vect{\linVelElem}}

\newcommand{\angVelElem}{\omega}

\newcommand{\angVel}{\vect{\angVelElem}}
\newcommand{\linCtrlElem}{u}
\newcommand{\linCtrl}{\vect{\linCtrlElem}}
\newcommand{\linCtrlMax}{\overline{\linCtrlElem}}
\newcommand{\dt}{\delta t}

\newcommand{\state}{\vect{x}}

\newcommand{\yawRate}{\angVelElem}

\newcommand{\track}{\vect{\tau}}
\newcommand{\trackWidth}{w_{\track}}
\newcommand{\trackParam}{s}
\newcommand{\trackLength}{l_{\track}}
\newcommand{\nTours}{N}

\newcommand{\trackTan}{\vect{t}}
\newcommand{\trackNorm}{\vect{n}}
\newcommand{\trackCurv}{\kappa}
\newcommand{\strategy}{\vect{\theta}}
\newcommand{\objective}{f}
\newcommand{\objFinal}{\trackParam}
\newcommand{\bestresponse}{\objFinal^*}
\newcommand{\bestMap}{\mathcal{R}}
\newcommand{\strategySpace}{\Theta}

\newcommand{\inEqConJointEl}{\gamma}

\newcommand{\inEqConJoint}{\vect{\inEqConJointEl}}
\newcommand{\eqConOwn}{\vect{h}}
\newcommand{\inEqConOwn}{\vect{g}}

\newcommand{\lagMultInEq}{\vect{\mu}}
\newcommand{\lagMultEqOwn}{\vect{\lambda}}
\newcommand{\lagMultInEqOwn}{\vect{\nu}}
\newcommand{\lagMult}{\mu}
\newcommand{\multGain}{\alpha}
\newcommand{\residual}{r}

\newcommand{\skewZ}{\matr{S}}

\newcommand{\fov}{\alpha}

\newcommand{\videourl}{\url{https://youtu.be/cJW1RysDKDg}}

\newtheorem{lemma}{Lemma}

\newtheorem{theorem}{Theorem}

\newlist{expcases}{enumerate}{1} 
\setlist[expcases]{label= case~\Roman*:, ref=\Roman*,align=left}
\crefname{expcasesi}{case}{cases}
\Crefname{expcasesi}{Case}{Cases}

\crefname{equation}{}{}
\Crefname{equation}{Equation}{Equations}
\crefmultiformat{equation}{(#2#1#3)}%
  { and~(#2#1#3)}{, (#2#1#3)}{ and~(#2#1#3)}
\Crefmultiformat{equation}{Equations (#2#1#3)}%
  { and~(#2#1#3)}{, (#2#1#3)}{ and~(#2#1#3)}
\crefname{figure}{Fig.}{Figs.}
\Crefname{figure}{Figure}{Figures}
\crefname{part}{Part}{Parts}
\Crefname{part}{Part}{Parts}
\crefname{chapter}{Chapt.}{Chapts.}
\Crefname{chapter}{Chapter}{Chapters}
\crefname{section}{Sect.}{Sects.}
\Crefname{section}{Section}{Sections}
\crefname{subsection}{Sect.}{Sects.}
\Crefname{subsection}{Section}{Sections}
\crefname{subsubsection}{Sect.}{Sects.}
\Crefname{subsubsection}{Section}{Sections}
\crefname{appsec}{Appendix}{Appendices}
\Crefname{appsec}{Appendix}{Appendices}
\crefname{appendix}{Appendix}{Appendices}
\Crefname{appendix}{Appendix}{Appendices}
\crefname{subappendix}{Appendix}{Appendices}
\Crefname{subappendix}{Appendix}{Appendices}
\crefname{lemma}{Lemma}{Lemmas}
\Crefname{lemma}{Lemma}{Lemmas}
\crefname{assumption}{Assumption}{Assumptions}
\Crefname{assumption}{Assumption}{Assumptions}
\crefname{remark}{Remark}{Remarks}
\Crefname{remark}{Remark}{Remarks}
\crefname{theorem}{Theorem}{Theorems}
\Crefname{theorem}{Theorem}{Theorems}
\crefname{prop}{Prop.}{Props.}
\Crefname{prop}{Proposition}{Propositions}
\crefname{prob}{Problem}{Problems}
\Crefname{prob}{Problem}{Problems}
\crefname{constr}{Constraint}{Constraints} 
\Crefname{constr}{Constraint}{Constraints}
\crefname{algorithm}{Algorithm}{Algorithms}
\Crefname{algorithm}{Algorithm}{Algorithms}
\crefname{ALC@unique}{step}{steps}
\Crefname{ALC@unique}{Step}{Steps}

\begin{document}

\maketitle
\thispagestyle{empty}
\pagestyle{empty}

\begin{abstract}
To be successful in multi-player drone racing, a player must not only follow the race track in an optimal way, but also compete with other drones through strategic blocking, faking, and opportunistic passing while avoiding collisions.
Since unveiling one's own strategy to the adversaries is not desirable, this requires each player to independently \emph{predict} the other players' future actions.
Nash equilibria are a powerful tool to model this and similar multi-agent coordination problems in which the absence of communication impedes full coordination between the agents.
In this paper, we propose a novel receding horizon planning algorithm that, exploiting sensitivity analysis within an iterated best response computational scheme, can approximate Nash equilibria in real time.
We also describe a vision-based pipeline that allows each player to estimate its opponent's relative position.
We demonstrate that our solution effectively competes against alternative strategies in a large number of drone racing simulations.
Hardware experiments with onboard vision sensing prove the practicality of our strategy.
\end{abstract}

\section{SUPPLEMENTARY MATERIAL}
Video of the experiments: \videourl{}.

\section{INTRODUCTION}

Drone racing has recently become a popular sport with international competitions being held regularly and attracting a growing public~\cite{wdp}.
In these races, human pilots directly control the UAVs through a radio transmitter while receiving a first-person-view live stream from an onboard camera.
Human racers need years of training to master the advanced navigation and control skills that are required to be successful in this sport.
Many of these skills would certainly prove useful for a robot to safely and quickly move through a cluttered environment in, for example, a disaster response scenario.
For this reason, drone racing has attracted a significant interest from the scientific community, which led to the first autonomous drone racing competition being held during the IROS 2016 international conference~\cite{MooSunBalKim17}.

%

Most of the past research has focused on a time trial style of racing: a single robot must complete a racing track in the shortest amount of time. This scenario poses a number of challenges in terms of dynamic modeling, on-board perception, localization and mapping, trajectory generation and optimal control. Impressive results have been obtained in this context not only for autonomous UAVs~\cite{SunChoLeeLeeShi18}, but also for a variety of different platforms, such as cars~\cite{KaSuGe16,WilDreGolRehThe16,WilWagGolDreRehBooThe17} motorcycles~\cite{motobot}, and even sailboats~\cite{StPr08}.

Much less attention, on the other hand, has been devoted to the more classical multi-player style of racing that we address in this paper, sometimes called rotocross among drone racing enthusiasts. In addition to the aforementioned challenges, this kind of race also requires direct competition with other agents, incorporating strategic blocking, faking, and opportunistic passing while avoiding collisions.
Multi-player drone racing is then also a good testing ground for developing and testing more widely applicable non-cooperative multi-robot planning strategies.

\begin{figure}[t]
\centering
\includegraphics[width=\columnwidth]{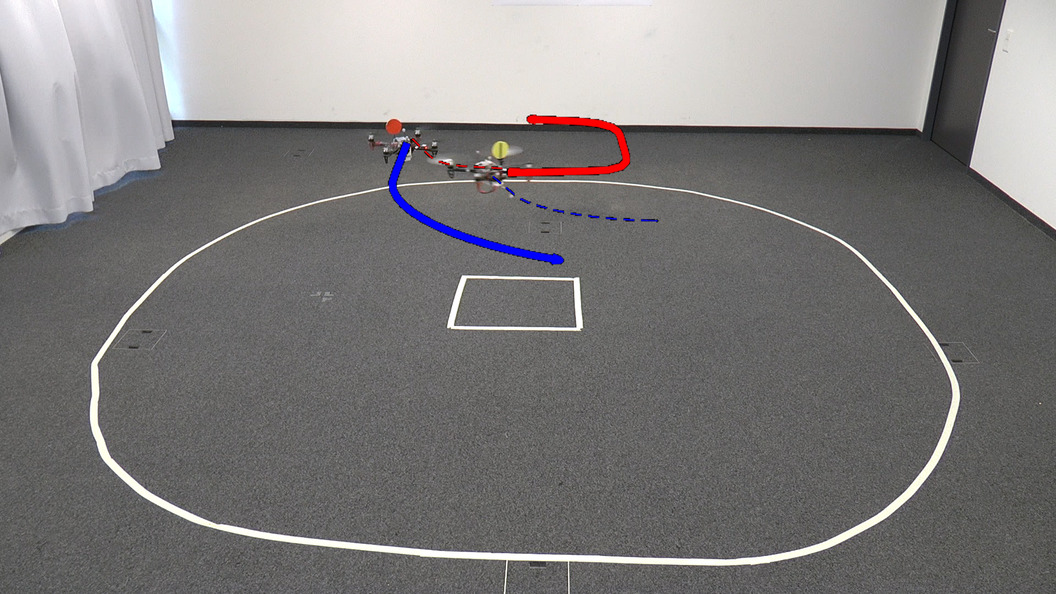}
\caption{Drone racing experiment.}
\label{fig:exp}
\end{figure}

Once the robots leave the protected and isolated environment of a factory floor, it becomes crucial for them to be able to safely and effectively interact with other robots and, perhaps more importantly, with human beings. In particular, they must avoid collisions that could compromise their functioning or cause injuries to humans.
In some cases, safety can be ensured by designing specific (and common) policies for the agents or by relying on communication. In a drone race, however, it would not be reasonable to impose a common control policy to all players or even to require players to unveil their future actions to the opponents. Another example, in this context, is autonomous driving in presence of human-driven cars. Indeed, assuming complex human-robot communication, beyond some basic form of signaling, is not realistic.

In order to guarantee safety in absence of communication, robots must be able to \emph{predict} what the other robotic/human agents will do and act consequently. To guarantee performance and robustness, the models employed for such predictions must also consider the reactive behavior of other agents to the robots own actions induced, for example, by reciprocal collision avoidance constraints.

In this paper, we consider two-player autonomous drone racing as a practical application to develop effective strategies for scenarios involving multiple rational agents that: $i$) do not communicate their policies to each other, $ii$) know each other's goals and constraints, $iii$) behave reactively in order to avoid collisions with other agents. We believe that game theory~\cite{BasOls98} is the most appropriate tool to model scenarios such as these.

Motivated by the success obtained by Model Predictive Control (MPC) in the development of real-time optimal control schemes, we apply similar receding horizon control strategies in the context of multi-player drone racing.
Differently from a standard MPC planner, however, our strategy also takes into account other agents reactions to the ego agent actions. We achieve this by employing an iterated best response computation scheme: each player alternatively solves an optimal control problem for \emph{each} player while keeping the other player's strategy constant. In addition to this, in order to fully capture and exploit the effects of the collision avoidance constraints, we also use sensitivity analysis to approximate the effects of one player's actions on its opponent's cost.

Despite the fact that Nash equilibria are often difficult to achieve or verify in dynamic games, we prove that, if our algorithm converges, the output satisfies necessary conditions for a Nash equilibrium. In practice, we find that the algorithm does converge, providing a theoretical foundation for our technique.  The algorithm also runs in real time, at 20Hz, on standard hardware.

This work focuses on the competition and interaction between the drones, not on the perception and navigation of the race course or environment.  Hence, we assume that each drone has a prior map of the race course, and has access to its own position (e.g.~with GPS or an external motion capture system).  However, each drone estimates the relative position of its competitor with an on board monocular camera.

We demonstrate the effectiveness of our approach in a large number of simulations in which our planner competes against multiple alternative strategies. We also prove the applicability of our method in hardware experiments with two quadrotor robots racing on an indoor circular course.

The rest of the paper is organized as follows. First, in \cref{sec:previous}, we review some of the existing relevant literature. Then, in \cref{sec:formulation}, we model the drone racing problem and introduce the associated sensing and control constraints. Subsequently, in~\cref{sec:game}, we formulate the position control problem as a Nash equilibrium search and we detail the numerical methods used to obtain real time solutions. We then describe our orientation control strategy. Then, in~\cref{sec:estimation}, we provide details on the algorithms used for estimating both the ego robot and the opponent positions. In \cref{sec:experiments} we report simulation and experimental results obtained by letting our method compete against alternative ones. Finally, in~\cref{sec:conclusion}, we conclude the paper and outline future extensions.

\section{PREVIOUS WORKS}\label{sec:previous}

In this section, we give a brief overview of some of the literature relevant to this work first in the context of single-robot motion planning, then in the multi-robot setting. We discuss both classical approaches and more recent results. In particular, we focus our attention on approaches exploiting results from game theory using either a Stackelberg or a Nash information patterns, which are more closely related to our work.

\subsection{Single-robot planning}
A number of effective solutions for motion planning in presence of both \emph{static} and \emph{dynamic} obstacles have been proposed in the past. Some classical works use artificial potential fields~\cite{MacKle1985,Kha1986}, geometric approaches~\cite{FioShi1998} or sampling-based methods~\cite{FraDahFer02,KarFra11}.
More recently, also thanks to the availability of efficient numerical optimization schemes, a number of Model Predictive Control (MPC) and Reinforcement Learning (RL) approaches have also been proposed~\cite{LiDoMo15,WilDreGolRehThe16,WilWagGolDreRehBooThe17}.
More specifically to our application, the authors of~\cite{SunChoLeeLeeShi18} ranked first in the IROS 2016 drone racing competition exploiting an optical flow sensor and a direct visual servoing control scheme.

Most of these works rely on simple ``open-loop'' models to predict the obstacle motion. In many situations, however, obstacles behave in a \emph{reactive} way. A human, for example, will in turn actively avoid collision with the controlled robot and, thus, her/his motion will be strongly affected by that of the robot. This reactiveness creates a ``loop closure'' which, if not properly managed, can induce oscillatory effects sometimes referred to as \emph{reciprocal dances}~\cite{FueChi2000}.

\subsection{Cooperative multi-robot control}
Impressive results have been obtained by relying on \emph{communication} to coordinate multiple robots in a navigation context or, more in general, to realize a common task~\cite{OlfFaxMur2007,AndYuFidHen2008}. In other cases, communication is achieved more implicitly by an exchange of forces~\cite{WanSch2016}. Some works remove the communication layer but rely on a common (or at least known) set of motion policies to achieve cooperation among multiple agents~\cite{FioShi1998,CaiYanZhuLia2007,HofTom08,vanGuyLinMan2011,ZhoWanBanSch2017}.

\subsection{Game theoretic control using a Stackelberg information pattern}
While broadly used in economics and social science, game theory has not yet attracted, in our opinion, a sufficient interest from the robotics community, mostly due to the computational complexity typically associated with these methods.

Some interesting results have been obtained applying game theoretic concepts to robust $H^\infty$ optimal control design (see~\cite{BaBe08} for a recent review on the topic). The disturbance acting on a system can be modeled as an antagonistic agent that explicitly aims at minimizing the system performance thus giving rise to a zero-sum differential game~\cite{BasOls98}.
Similar models have also been employed to calculate the so called \emph{reachable set} of a system: the set of states from which there exists at least one disturbance that brings the system to a dangerous state~\cite{MitBayTom2005}. 
Finding the reachable sets usually requires the integration of the so-called Hamilton-Jacobi-Isaacs (HJI) partial differential equations. Apart from simple cases, the computational complexity of these methods often limits their use to offline implementations. Real-time solutions can be obtained, in some cases, by leveraging suitable approximations~\cite{GiHoHuViTo11}.


More recently, similar approaches have also been employed in the context of autonomous driving. In~\cite{SadSasSesDra2016b}, for example, the interaction between an autonomous car and a human driven one is modeled as a Stackelberg game: the human is assumed to know in advance what the autonomous car will do and to respond optimally according to an internal cost function. This results in a nested optimization problem that can be exploited to control the human motion~\cite{SadSasSesDra2016a} or to reconstruct the cost function driving his/her actions~\cite{SadSasSesDra2016b}.

\subsection{Game theoretic control using a Nash information pattern}
Giving the other players some information advantage can, in general, improve the robustness of the system. However, in many applications such as drone racing and autonomous driving, no agent would have any information advantage with respect to the others. For this reason, we believe that Stackelberg information models could result in overly conservative actions. A more realistic model is that of Nash equilibria which, instead, assume a fully symmetric information pattern.


A very recent paper~\cite{DreGer2017} proposes a control algorithm for coordinating the motion of multiple cars through an intersection exploiting generalized Nash equilibria. The numerical resolution is in the order of several seconds which is close to real time but still not sufficient for the approach to be used for online control.

In the context of car racing, the authors of~\cite{LinLyg17} investigate both Stackelberg and Nash equilibria. Computational performance close to real time, however, can only be obtained in a simplified scenario in which only one of the two players avoids collisions. In addition to this, the authors also discuss the importance of exploiting blocking behaviors. However, while in our work these behaviors naturally emerge from the use of sensitivity analysis, in~\cite{LinLyg17} these are hardcoded in the cost function optimized by the players.

The main limiting factor for applying game theory more widely to robotic control problems seems to lie in the associated computational complexity.
We believe, however, that game theory can still be used as an \emph{inspiration} for guiding the design of effective and computationally efficient heuristics.
This paper is a first step in this direction.

\section{PRELIMINARIES}\label{sec:formulation}

Consider two quadrotor UAVs competing against each other in a drone racing scenario.
In order to simplify the high level control strategy, we will assume that the robots fly at a constant altitude with simplified holonomic dynamics given by
\begin{equation}\label{eq:dynamics}
\begin{bmatrix} \dot{\pos}_i \\ \dot{\yaw}_i \end{bmatrix} =
\begin{bmatrix} \rotMat{}{i} & 0 \\
				\zeros{2}{1} &              1  \end{bmatrix}
\begin{bmatrix} \linVel_i \\ \yawRate_i \end{bmatrix},
\end{equation}
where ${\pos}_i\in\real{2}$ is the robot horizontal position in the world frame, $\rotMat{}{i}=\rotMat{}{z}(\yaw_i)\in\real{2\times2}$ represents the rotation matrix associated to the robot yaw $\yaw_i\in\setS{1}$, and $\linVel_i\in\real{2}$ and $\yawRate_i\in\real{}$ are the \emph{body-frame} linear velocity and angular rates, which are assumed to be known and controllable.
Given~\cref{eq:dynamics}, the robot state is $\state_i = (\pos_i, \yaw_i) \in \real{2} \times \setS{1}$ and it is assumed locally available, e.g.~using onboard GPS and compass sensors.

Due to limitations of onboard actuators, the robots linear velocities are limited, i.e.
\begin{equation*}
\norm{\linVel_i} = \norm{\dot{\pos}_i}\le \linVelMax_i\in\real{+}.
\end{equation*}

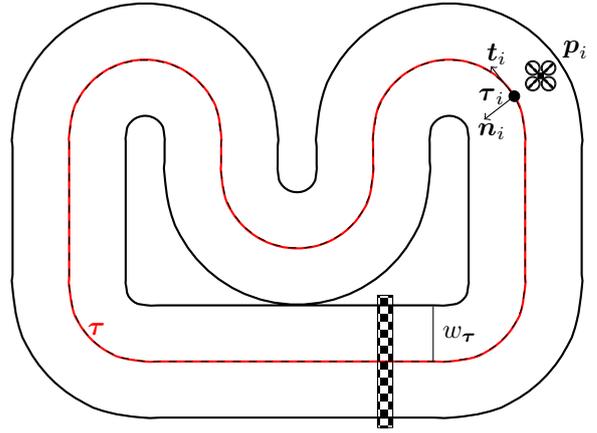
\begin{figure}[tb]
\centering
\pgfdeclarepatternformonly{my_checkerboard}{\pgfpointorigin}{\pgfqpoint{2mm}{2mm}}{\pgfqpoint{2mm}{2mm}}%
{
  \pgfpathrectangle{\pgfpointorigin}{\pgfqpoint{1mm}{1mm}}
  \pgfpathrectangle{\pgfqpoint{1mm}{1mm}}{\pgfqpoint{1mm}{1mm}}
  \pgfusepath{fill}
}

 \begin{tikzpicture}[scale=1]

	\def\tw{0.73cm};
	\coordinate (T) at (0,0);
	\newcommand{\vertT}
  {(+1.158,-1.882),(+1.369,-1.882),(+1.579,-1.882),(+1.790,-1.882),(+2.000,-1.882),(+2.309,-1.833),(+2.482,-1.758),(+2.638,-1.652),(+2.771,-1.519),(+2.877,-1.364),(+2.951,-1.191),(+3.000,-0.882),(+3.000,-0.776),(+3.000,-0.671),(+3.000,-0.566),(+3.000,-0.461),(+3.000,-0.355),(+3.000,-0.250),(+3.000,-0.145),(+3.000,-0.040),(+3.000,+0.066),(+3.000,+0.171),(+3.000,+0.276),(+3.000,+0.381),(+3.000,+0.487),(+3.000,+0.592),(+3.000,+0.697),(+3.000,+0.802),(+3.000,+0.908),(+3.000,+1.013),(+3.000,+1.118),(+2.951,+1.427),(+2.857,+1.635),(+2.717,+1.815),(+2.541,+1.960),(+2.336,+2.060),(+2.114,+2.112),(+1.886,+2.112),(+1.664,+2.060),(+1.460,+1.960),(+1.283,+1.815),(+1.143,+1.635),(+1.049,+1.427),(+1.000,+1.118),(+1.000,+0.993),(+1.000,+0.868),(+1.000,+0.743),(+1.000,+0.618),(+0.951,+0.309),(+0.856,+0.102),(+0.717,-0.079),(+0.540,-0.223),(+0.336,-0.324),(+0.114,-0.375),(-0.114,-0.375),(-0.336,-0.324),(-0.541,-0.223),(-0.717,-0.079),(-0.857,+0.102),(-0.951,+0.309),(-1.000,+0.618),(-1.000,+0.743),(-1.000,+0.868),(-1.000,+0.993),(-1.000,+1.118),(-1.049,+1.427),(-1.143,+1.635),(-1.283,+1.815),(-1.460,+1.960),(-1.664,+2.060),(-1.886,+2.112),(-2.114,+2.112),(-2.336,+2.060),(-2.541,+1.960),(-2.717,+1.815),(-2.857,+1.635),(-2.951,+1.427),(-3.000,+1.118),(-3.000,+1.013),(-3.000,+0.908),(-3.000,+0.802),(-3.000,+0.697),(-3.000,+0.592),(-3.000,+0.487),(-3.000,+0.381),(-3.000,+0.276),(-3.000,+0.171),(-3.000,+0.066),(-3.000,-0.040),(-3.000,-0.145),(-3.000,-0.250),(-3.000,-0.355),(-3.000,-0.461),(-3.000,-0.566),(-3.000,-0.671),(-3.000,-0.776),(-3.000,-0.882),(-2.951,-1.191),(-2.877,-1.364),(-2.771,-1.519),(-2.638,-1.652),(-2.482,-1.758),(-2.309,-1.833),(-2.000,-1.882),(-1.790,-1.882),(-1.579,-1.882),(-1.369,-1.882),(-1.158,-1.882),(-0.948,-1.882),(-0.737,-1.882),(-0.527,-1.882),(-0.316,-1.882),(-0.106,-1.882),(+0.105,-1.882),(+0.316,-1.882),(+0.526,-1.882),(+0.737,-1.882),(+0.947,-1.882)}

  \newcommand{\normal}[2]{
    ($(#1)!\tw!90:(#2)$)
  }
  \newcommand{\tangent}[2]{
    ($(#1)!\tw!(#2)$)
  }

  \edef\points{}
  \foreach \point [count=\i] in \vertT {
      \pgfmathsetmacro\ii{int(\i-1)}
      \node[coordinate] (T\ii) at \point {} ;
      \xdef\points{(T\ii) \points}
  }

  \foreach \coord [count=\i] in \vertT {
    \pgfmathsetmacro\ii{int(\i-1)}
    \coordinate [at=\coord, name=tmp];
    \coordinate (T\ii) at ($(T)+(tmp)$);
  }

  \foreach \coord [count=\i] in \vertT {
    \pgfmathsetmacro\ii{(\i-1)}
    \pgfmathsetmacro\ip{int(mod(\ii+1,117))}
    \pgfmathsetmacro\ii{int(\i-1)}
    \coordinate [at=\normal{T\ii}{T\ip}, name=Tin\ii];
    \coordinate [at=($(T\ii)!\tw!-90:(T\ip)$), name=Tout\ii];
  }

  \draw [double=white,double distance=2*\tw,thick,tension=.8] plot [smooth cycle] coordinates {\points};
	\draw [red, thick] plot [smooth cycle,tension=0.8] coordinates {\points};
	\node [red] at (T100) [above] {$\track$};

  \draw [dashed] plot coordinates {\points};

  \draw (T3) -- (Tin3) node [pos=.5,right] {$\trackWidth$ \pgfmathparse{dim(\vertT)}};

	\node [quadcopter top,minimum width=.4cm,label={[label distance=.2cm]30:$\pos_i$}] at ($(T32)!-.6!(Tin32)$) {};
	\draw [fill] (T32) circle (2pt) node [left]{$\track_i$};

	\draw [->] (T32) -- ($(T32)!.5cm!90:(T33)$) node [below, near end] {$\trackNorm_i$};
	\draw [->] (T32) -- ($(T32)!.5cm!(T33)$) node [above,near end] {$\trackTan_i$};

  \def\mypath{(0,0) -- +(0,1) arc (180:0:1.5cm) -- +(0,-1)}
  \draw [pattern=my_checkerboard] ($(T0)-(1mm,1.2*\tw)$) rectangle ($(T0)+(1mm,1.2*\tw)$);

 \end{tikzpicture}
\caption{Representation of the race track used for the simulations. The track is parameterized by its center line $\track$ and its half width $\trackWidth$. Given the current robot position $\pos_i$, we can define a local track frame with origin $\track_i$ as the closest point to $\pos_i$ and with $\trackTan$ and $\trackNorm$ being the local tangent and normal vectors to the track in $\track_i$.}
\label{fig:track}
\end{figure}

The race track center line is defined by a twice continuously differentiable immersed plane closed curve $\track$ (see~\cref{fig:track}). For such a curve, there exists an arch-length parameterization
\begin{equation*}
\track: [0,\trackLength]\mapsto \real{2} \text{, with } \track(0) = \track(\trackLength)
\end{equation*}
where $\trackLength$ is the total length of the track.
Moreover, one can also define a local signed curvature $\trackCurv$ and unit tangent and normal vectors ($\trackTan$ and $\trackNorm$ respectively) as follows
\begin{align}
\trackTan &=  \track' \label{eq:tangent}\\ 
\trackNorm\trackCurv &= \track''.  \label{eq:normal}
\end{align}

To remain within the boundaries of the track, the robot's distance from the track center line must be smaller than the (constant) track width $\trackWidth\in\real{+}$, i.e.
\begin{equation*}
\abs{\trackNorm(\trackParam_i)^T [\pos_{i}-\track(\trackParam_i)]} \le \trackWidth
\end{equation*}
where $\trackParam_i\in[0,\trackLength]$ is the robot position along the track, i.e.~the arch length of the point on the track that is closest to $\pos_{i}$
\begin{equation}\label{eq:track_pos}
\trackParam_i(\pos_i)=\arg\min_\trackParam{\tfrac{1}{2}\norm{\track(\trackParam)-\pos_i}^2}.
\end{equation}

In order to avoid potential collisions, each robot always maintains a minimum distance $\distMax_i\in\real{+}$ with respect to its opponent, i.e.
\begin{equation}\label{eq:coll_avoid}
\norm{\pos_{i}-\pos_{j}} \ge \distMax_i.
\end{equation}
Note that here, as well as in the rest of the paper, we always use $i$ ($=1$ or $2$) to refer to a generic robot and $j$ ($=2$ or $1$ respectively) to refer to its opponent.

Each UAV is also equipped with an onboard calibrated monocular camera that can detect and track a spherical marker attached to the opponent body, provided that this latter is in the camera field of view.
Assuming that the radius of the spherical target is known, they can retrieve the relative position of the opponent expressed in their local body-frame, i.e.
\begin{equation}\label{eq:bearing}
\pos_{ij} = \rotMat{}{i}^T\left(\pos_{j}-\pos_{i}\right).
\end{equation}
Fusing~\cref{eq:bearing} with their ego state estimate, each robot can then estimate the world frame position of its opponent.

Since the robot cameras have a limited field of view, we need to impose a visibility constraint for measurement~\cref{eq:bearing} to be available. We assume that the size of the spherical markers mounted on the robots is such that, even when the two robots are at a minimum distance $\distMax_i$ from each other, a robot marker is entirely visible by its opponent camera, provided that the following condition is satisfied:
\begin{equation}\label{eq:fov_constr}
\frac{\pos_{ij}^T}{\norm{\pos_{ij}}}\ex \defas \bearing_{ij}^T\ex \ge \cos(\fov)
\end{equation}
where $ \bearing_{ij}$ is a relative bearing vector, $\fov$ is the camera field of view and $\ex=(1,0)$ is the optical axis which is assumed, without loss of generality, to be aligned with the $x$-axis of the robot body frame.


Since we exploit a receding horizon control approach, the objective for each robot is to have a more advanced position along the track, with respect to the opponent, at the end of the planning horizon $T$. The final position is given by:
\begin{equation*}
\dist_{i} = \nTours_i \trackLength + \trackParam_i(\pos_i(t+T)) 
\end{equation*}
where $\nTours_i$ is the number of completed track loops and $\trackParam_i$ is computed as in~\cref{eq:track_pos}. Neglecting the constant terms, the objective function of player $i$ is then to maximize the difference
\begin{equation}\label{eq:objective}
\objective_i = \trackParam_i(\pos_i(t+T))-\trackParam_j(\pos_j(t+T)).
\end{equation}


Because of the collision avoidance constraints~\cref{eq:coll_avoid}, in order to calculate its optimal trajectory, each robot needs access to its opponent's strategy. However, since the robots are competing against each other, we do not expect them to share/communicate their plans. Instead, each robot needs to model the opponent and predict its actions.
We believe that game theory~\cite{BasOls98} is the correct framework to describe this non-cooperative scenario. In particular, drone racing can be seen as a \emph{zero-sum} differential game because clearly from~\cref{eq:objective} one has $\objective_1 + \objective_2 = 0$.

Since the cost function~\cref{eq:objective} only depends on the robots' positions and the constraints on the robot positions and yaw angles can be separated, we perform the planning for the robots' position and yaw angles separately: first we apply a game theoretic approach to calculate optimal control inputs for the translational part of the robot dynamics; then we calculate the yaw angle control in such a way that the visibility constraints~\cref{eq:fov_constr} remain satisfied at all times given the planned/expected translational motions.

\section{GAME THEORETIC FORMULATION}\label{sec:game}

In this section we address the planning problem for the translational component of the robot state, i.e.~the first row of~\cref{eq:dynamics}.
Since the robots know their relative positions and their state with respect to the world frame, we can rewrite the optimization problem in world frame coordinates. By doing this, the robots dynamics further simplify to $\dot{\pos}_i = \linCtrl_i$. We also note that, since the second term in~\cref{eq:objective} does not depend on player $i$'s actions, this can be neglected without changing the set of optimal solutions for player $i$.

To make the problem tractable, we discretize the planning horizon and we assume piecewise constant control inputs for both players, i.e.~$\linCtrl_i(t) = \linCtrl_i^k/\delta t = \textrm{const }\forall t\in[t_0,t_0+k\delta t)$ where $\delta t$ is a constant sampling interval. Defining $\strategy_i=(\pos_i^1,\dots,\,\pos_i^N,\linCtrl_i^1,\dots,\linCtrl_i^N)$ and $\linCtrlMax_i = \linVelMax_i \dt$, the problem can then be rewritten as
\begin{subequations}
\begin{align}
\max_{\strategy_i}\,\,& \trackParam_i(\pos_i^N)
 \label{eq:obj} \\
\text{s.t. }
\quad &\pos_{i}^{k} = \pos_{i}^{k-1}+\linCtrl_{i}^{k}\label{eq:cons_dyn}  \\
&\norm{\pos_j^k-\pos_i^k} \ge \distMax_i \label{eq:obs_avoid}\\
&\abs{{\trackNorm(\pos_i^k)}^T [\pos_i^k-\track(\pos_i^k)]} \le \trackWidth \label{eq:cons_track} \\
&\norm{\linCtrl_i^k} \le \linCtrlMax_i. \label{eq:cons_max_vel}
\end{align}\label{eq:game_discrete}
\end{subequations}

For simplicity of notation let us rewrite problem~\cref{eq:game_discrete} in a more compact and general form
\begin{subequations}
\begin{align}
\max_{\strategy_i}\,\,& {\trackParam_i(\strategy_i)} \label{eq:pli_optim_prob_cn0}\\
\text{s.t. }
&\eqConOwn_i(\strategy_i) = 0 \label{eq:pli_optim_prob_cn1}\\
&\inEqConOwn_i(\strategy_i) \le 0 \label{eq:pli_optim_prob_cn2}\\
&\inEqConJoint_i(\strategy_i,\strategy_j) \le 0 \label{eq:pli_optim_prob_cn4}
\end{align}\label{eq:pli_optim_prob}
\end{subequations}
where:
\begin{itemize}
\item $\eqConOwn_i$ represents the equality constraints~\cref{eq:cons_dyn} involving a single player;
\item $\inEqConOwn_i$ represents the inequality constraints~\cref{eq:cons_track,eq:cons_max_vel} involving a single player;
\item $\inEqConJoint_i$ represents the inequality constraints~\cref{eq:obs_avoid} involving both players.
\end{itemize}
Let us also define $\strategySpace_i \subseteq \real{4N}$ as the space of admissible strategies for player $i$, i.e.~strategies that satisfy~\crefrange{eq:pli_optim_prob_cn1}{eq:pli_optim_prob_cn4}. Note that, due to~\cref{eq:obs_avoid}, one has $\strategySpace_i=\strategySpace_i(\strategy_j)$, i.e.~the strategy of one player determines the set of admissible strategies of its opponent and, as a consequence, can influence this latter's behavior.

In a game, the concept of an \emph{optimal solution} loses meaning because, in general, and especially in a zero-sum game, it is not possible to find a pair of strategies $(\strategy_1,\strategy_2)$ that maximize the cost function of both agents simultaneously. On the other hand, various types of \emph{equilibria} can be defined depending on the degree of cooperation between the agents and the information pattern of the game (see~\cite{BasOls98} for a complete description of the possible alternatives).
In this work, in particular, we exploit the concept of \emph{Nash equilibria}, which, by modeling a perfectly symmetric information pattern, do not induce overly optimistic or conservative behaviors.

A Nash equilibrium is a strategy profile $(\strategy_1^*,\strategy_2^*)\in \strategySpace_1\times\strategySpace_2$ such that no player can improve its own outcome by unilaterally changing its own strategy, i.e.
\begin{equation}\label{eq:nash}
\strategy_i^* = \arg\max_{\strategy_i\in\strategySpace_i(\strategy_j^*)} \trackParam_i(\strategy_i)
\end{equation}
An alternative definition of Nash equilibria can be given by defining a \emph{best reply map}
\begin{equation*}
\bestMap_i(\strategy_j) = \left\{ \strategy_i \in \strategySpace_i(\strategy_j) \st \trackParam_i(\strategy_i) =  \bestresponse_i(\strategy_j) \right\}
\end{equation*}
where
\begin{equation}
\bestresponse_i(\strategy_j) = \max_{\strategy_i\in\strategySpace_i(\strategy_j)} {\trackParam_i(\strategy_i)}\label{eq:best_response_return}
\end{equation}
is player $i$'s \emph{best-response return} to player $j$'s strategy $\strategy_j$.
One can show that a Nash equilibrium is a fixed point of the best reply map, i.e.~such that $\strategy_i^* \in \bestMap_i(\strategy_j^*)$.

Unfortunately, since problem~\cref{eq:game_discrete} is not convex due to~\cref{eq:obs_avoid}, in general multiple Nash equilibria may exist (e.g.~left vs right side overtaking). Additionally, computing the \emph{exact} value of some Nash equilibrium, generally requires numerical algorithms whose computational complexity makes them still not well suited for online robot control. Therefore, the next section describes an iterative algorithm that allows to \emph{approximate} Nash equilibria in real time.

\subsection{Numerical resolution of Nash equilibria}\label{sec:num_game}

In order to approximate Nash equilibria in real time, we use an iterated best response algorithm (IBR). Starting from an initial guess of the Nash equilibrium strategy profile, we update each player's strategy, alternatively, to the best-response to the current opponent's strategy. This is done by solving a standard optimization problem in which one player strategy is allowed to change while the opponent's one is kept constant. Intuitively, if the resulting sequence of strategy profiles converges, it follows that each player is best-responding to its opponent. If this is the case, then no profitable unilateral change of strategy exists as required by the Nash equilibrium definition~\cref{eq:nash}.

From our perspective, in the aim of developing a real time planner, IBR has the advantage that one can finely tune how much each player takes into account the reactivity of its opponent. By limiting the number of iterations per planning step, one can cover a spectrum of behaviors ranging from a very efficient, but naive, classical optimal control problem (with a fixed guess for the opponent strategy) to a more computationally expensive but fully game theoretic approach.

Unfortunately, a direct application of IBR to~\cref{eq:game_discrete} does not allow to fully capture the implications of the collision avoidance constraints~\cref{eq:obs_avoid}. As already mentioned, in fact, since player $i$ has no direct influence over the final position of player $j$ (i.e.~$\trackParam_j$), the second term in~\cref{eq:objective} can be neglected in~\cref{eq:game_discrete}. However, since player $j$ is calculating its strategy by solving an optimization problem similar to~\cref{eq:game_discrete}, due to the presence of the joint constraints~\cref{eq:obs_avoid}, player $i$ does have an effect on $\trackParam_j^*(\strategy_i^*)$ (see the counterpart of~\cref{eq:best_response_return} for player $j$). In other words, while player $i$ does not affect player $j$'s final position \emph{in general}, it does affect it at the Nash equilibrium.
To capture these effects, we propose to substitute~\cref{eq:pli_optim_prob_cn0} with the following cost function
\begin{equation*}
\trackParam_i(\strategy_i) - \multGain\bestresponse_j(\strategy_i)
\end{equation*}
where $\multGain \ge 0$ is a free parameter.

A closed form expression for $\bestresponse_j(\strategy_i)$ is hard to obtain. Inspired by~\cite{RaiEht2000}, we can, however, exploit sensitivity analysis to calculate a linear approximation around the current guess for the Nash equilibrium strategy profile.


Let us assume that, at the $l$-th iteration, a guess $\strategy_i^{l-1}$ for player $i$'s strategy is available to player $j$. Given this strategy for its opponent, player $j$ can solve the optimal control problem~\cref{eq:game_discrete} with $\strategy_i=\strategy_i^{l-1}$ (fixed). This step will result in a new best-responding strategy for player $j$, $\strategy_j^{l}$, with the associated payoff $\bestresponse_j(\strategy_i^{l-1})$. Assuming player $i$ is now given the opportunity to modify its own strategy, we are interested in characterizing the variations of $\bestresponse_j(\strategy_i)$ for $\strategy_i$ in the vicinity of $\strategy_i^{l-1}$ using a first-order Taylor approximation
\begin{equation}\label{eq:brr_approx}
  \bestresponse_j(\strategy_i) \approx \bestresponse_j(\strategy_i^{l-1}) + \evalin{\totDeriv{\bestresponse_j}{\strategy_i}}{\strategy_i^{l-1}}(\strategy_i-\strategy_i^{l-1}).
\end{equation}

Exploiting the Karush–Kuhn–Tucker (KKT) necessary optimality conditions associated to player $j$'s optimal control problem~\cref{eq:pli_optim_prob} one can prove the following result.
\begin{lemma}\label{lemma:sensitivity}
If $\bestresponse_j$ is the optimal value of an optimization problem obtained from~\cref{eq:pli_optim_prob} by exchanging subscripts $i$ and $j$, then
\begin{equation}\label{eq:sensitivity}
\evalin{\totDeriv{\bestresponse_j}{\strategy_i}}{\strategy_i^{l-1}} =
-\lagMultInEq_j^l \evalin{\parDeriv{\inEqConJoint_j}{\strategy_i}}{(\strategy_i^{l-1},\strategy_j^{l})}
\end{equation}
where $\strategy_j^{l}\in\bestMap_j(\strategy_i^{l-1})$ is the best-response of player $j$ to $\strategy_i^{l-1}$ and $\lagMultInEq_j^l$ is the row vector of Lagrange multipliers associated to the joint inequality constraints~\cref{eq:pli_optim_prob_cn4}.
\end{lemma}
\begin{proof}
A full discussion on sensitivity analysis can be found in~\cite{Fiacco83}. A brief proof, specific to the case at hand, is reported in \cref{sec:app_sensitivity}.
\end{proof}

Neglecting any term that is constant with respect to $\strategy_i$, we then propose that the ego vehicle solves the following optimization problem alternatively for itself and its opponent:
\begin{equation}\label{eq:final_iterative}
\max_{\strategy_i \in \strategySpace_i^l}
\trackParam_i(\strategy_i) + \multGain 
\lagMultInEq_j^{l} \evalin{\parDeriv{\inEqConJoint_j}{\strategy_i}}{(\strategy_i^{l-1},\strategy_j^{l})} \strategy_i
\end{equation}
where $\strategySpace_i^l$ respresents the space of strategies $\strategy_i$ that satisfy~\crefrange{eq:pli_optim_prob_cn1}{eq:pli_optim_prob_cn4} with $\strategy_j = \strategy_j^l$.

\begin{theorem}
If $\inEqConJoint_1(\strategy_1,\strategy_2) = \inEqConJoint_2(\strategy_1,\strategy_2)$ and the iterations converge to a solution $(\strategy_1^l,\strategy_2^l)$, then the strategy tuple $(\strategy_1^l,\strategy_2^l)$ satisfies the necessary conditions for a Nash equilibrium.\label{thm:nash}
\end{theorem}
\begin{proof}
See \cref{sec:app_nash}.
\end{proof}



In the drone racing scenario, in particular, using~\cref{eq:track_pos,eq:coll_avoid} after some straightforward calculation,~\cref{eq:final_iterative} reduces to
\begin{equation}\label{eq:game_final}
\max_{\strategy_i\in\strategySpace_i^l}{ \left[\arg\min_s{\tfrac{1}{2}{\norm{\track(s)-\pos_i^N}}^2} + \multGain\sum_{k=1}^{N}  \lagMult_j^{k,l}{\bearing_{ij}^{k,l}}^T\pos_i^{k} \right]}
\end{equation}
where
\begin{equation*}
\bearing_{ij}^{k,l}=\frac{\pos_j^{k,l}-\pos_i^{k,l-1}}{\norm{\pos_j^{k,l}-\pos_i^{k,l-1}}}.
\end{equation*}

To obtain a more intuitive interpretation of this result, let us assume that the track is linear and aligned to a unit vector $\trackTan$ so that the first term in~\cref{eq:game_final} can be rewritten as $\trackTan^T\pos_i^N$ (see \cref{sec:optim_ctrl} for details). Since player $i$ cannot modify the strategy of player $j$, the following problem has the same solutions as~\cref{eq:game_final}
\begin{equation}
\max_{\strategy_i\in\strategySpace_i^l}{ \trackTan^T \pos_{i}^{N} - \multGain\sum_{k=1}^{N} \lagMult_j^{k,l}{\bearing_{ij}^{k,l}}^T (\pos_j^{k,l}-\pos_i^{k}) }
\label{eq:game_linearized}
\end{equation}

We can then notice the following insightful facts.
First of all, if none of the collision avoidance constraints~\cref{eq:obs_avoid} were active in the $l$-th instance of problem~\cref{eq:game_discrete}, i.e.~if $\lagMult_j^{k,l}=0$, then~\cref{eq:game_linearized} reduces to~\cref{eq:game_discrete}. This has an intuitive explanation: if the collision avoidance constraints are not active, the optimal control problems for the two players are independent of each other and the original dynamic game reduces to a pair of classical optimal control problems. Interestingly, in this case, the only sensible strategy for a player is to advance as much as possible along the track.

The problem becomes much more interesting when the collision constraints are active ($\lagMult_j^{k,l}>0$). In this case, indeed, the cost function optimized in~\cref{eq:game_linearized} contains additional terms with respect to~\cref{eq:obs_avoid}. By inspecting these terms, one can easily notice that they have a positive effect on player $i$'s reward if robot $i$ reduces its distance from player $j$'s predicted position ($\pos_j^{k,l}$) along the direction of ${\bearing_{ij}^{k,l}}$. The intuition behind this is that, when the collision avoidance constraints are active, player $i$ can win the race by either going faster along the track or by getting in the way of player $j$, thus obstructing its motion along the path.

Isolating the last term in the summation, one can also rewrite the problem as
\begin{equation*}
\max_{\strategy_i\in\strategySpace_i^l}{ {\left(\trackTan+\multGain\lagMult_j^{k,N}{\bearing_{ij}^{k,N}}\right)}^T \pos_{i}^{N} +\multGain\sum_{k=1}^{N-1}  \lagMult_j^{k,l}{\bearing_{ij}^{k,l}}^T\pos_i^{k} }.
\end{equation*}
From this alternative expression it is clear that, depending on the value of $\multGain\lagMult_j^{k,N}$, player $i$ might actually find it more convenient to move its last position in the direction of player $j$ ($\bearing_{ij}^{k,N}$) rather than along the track ($\trackTan$). One can then also interpret the free scalar gain $\multGain$ as an \emph{aggressiveness} factor.
Using~\cref{eq:cons_dyn} one can also substitute $\pos_{i}^{N} = \pos_{i}^{n} + \sum_{k=n+1}^{N}\linCtrl_{i}^{k}$ and draw similar conclusions for any intermediate position $\pos_{i}^{n}$.

Note that player $i$ can exploit this effect only so long as it does not cause a violation of its own collision avoidance constraint~\cref{eq:obs_avoid}.

Before concluding this section, we want to stress the fact that, since the players do not communicate with each other, each of them must independently run the iterative algorithm described above and alternatively solve the optimization problem~\cref{eq:game_final} for themselves and for their opponent. 
In order to generate control inputs in real time, in our implementation we do not wait until convergence to a Nash equilibrium. Instead, we perform a constant number of iterations $L$, which can be set depending on the available computational resources. Since updating the opponent's strategy is only useful if this is exploited for recomputing a player's own strategy, we conclude each player's iterations with an extra resolution of its own optimal control problem. This also ensures that the resulting strategy profile satisfies the player's own constraints.

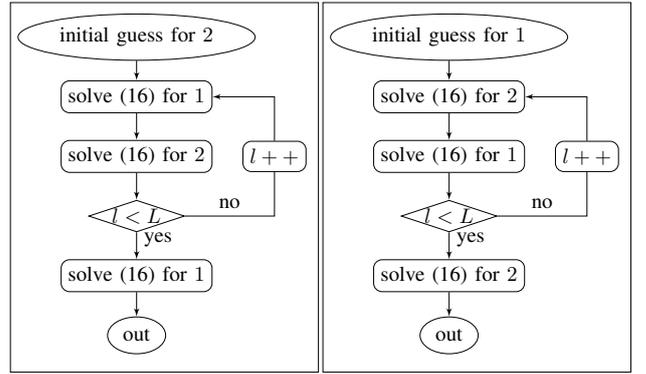
\begin{figure}
\centering
\resizebox{\columnwidth}{!}{
\tikzstyle{decision} = [diamond, draw, text badly centered, inner sep=0pt, aspect=3]
\tikzstyle{block} = [rectangle, draw, text centered, rounded corners]
\tikzstyle{line} = [draw, -latex']
\tikzstyle{cloud} = [draw, ellipse]

\begin{tikzpicture}[auto]
    \node [cloud] (init1) {initial guess for $2$};
    \node [block, below of=init1] (sol11) {solve~\cref{eq:game_final} for $1$};
    \node [block, below of=sol11] (sol12) {solve~\cref{eq:game_final} for $2$};
    \node [decision, below of=sol12] (decide1) {$l<L$};
    \node [block, below of=decide1] (sol11e) {solve~\cref{eq:game_final} for $1$};
    \node [cloud, below of=sol11e] (out1) {out};
    \node [block, right={.5cm of sol12}] (lpp) {$l++$};

	\draw ($(init1.north west)+(-.7cm,.3cm)$) rectangle ($(lpp.east |- out1.south)+(.2cm,-.3cm)$);

    \path [line] (init1) -- (sol11);
    \path [line] (sol11) -- (sol12);
    \path [line] (sol12) -- (decide1);
    \path [line] (decide1) -- node [near start] {yes} (sol11e);
    \path [line] (decide1) -| node [near start] {no} (lpp) |- (sol11);
    \path [line] (sol11e) -- (out1);

    \node [cloud] at (5.2cm,0) (init2) {initial guess for $1$};
    \node [block, below of=init2] (sol21) {solve~\cref{eq:game_final} for $2$};
    \node [block, below of=sol21] (sol22) {solve~\cref{eq:game_final} for $1$};
    \node [decision, below of=sol22] (decide2) {$l<L$};
    \node [block, below of=decide2] (sol21e) {solve~\cref{eq:game_final} for $2$};
    \node [cloud, below of=sol21e] (out2) {out};
    \node [block, right=.5cm of sol22] (lpp2) {$l++$};

 	\draw ($(init2.north west)+(-.7cm,.3cm)$) rectangle ($(lpp2.east |- out2.south)+(.2cm,-.3cm)$);

    \path [line] (init2) -- (sol21);
    \path [line] (sol21) -- (sol22);
    \path [line] (sol22) -- (decide2);
    \path [line] (decide2) -- node [near start] {yes} (sol21e);
    \path [line] (decide2) -| node [near start] {no} (lpp2) |- (sol21);
    \path [line] (sol21e) -- (out2);

\end{tikzpicture}
}
\caption{Flowchart representation of the iterative algorithm used by the two players to find approximate Nash equilibria.}
\label{fig:game_flow}
\end{figure}

As for the resolution of each player's optimal control problem~\cref{eq:game_final}, we use an iterative algorithm described in \cref{sec:optim_ctrl}.

\subsection{Numerical resolution of players' optimization}\label{sec:optim_ctrl}

The resolution strategy described in \cref{sec:num_game} relies on the assumption that, at each iteration $l$, an optimal solution to problem~\cref{eq:game_final} can be found given the current guess for the opponent strategy at the equilibrium.

Note that~\cref{eq:game_final} is a well posed problem. Indeed, due to the system dynamic constraints~\cref{eq:cons_dyn}, the input boundedness imposed by constraints~\cref{eq:cons_max_vel}, and assuming that the sampling time is finite, the set $\strategySpace_i^l$ is bounded. On the other hand $\strategySpace_i^l$ is also never empty because the solution $\strategy_i = (\pos_{i}^0,\dots,\pos_i^0,\zeros{2}{1},\dots,\zeros{2}{1})$ is always feasible, assuming that the robots do not start from a position that violates~\cref{eq:obs_avoid,eq:cons_track}.

Unfortunately, problem~\cref{eq:game_final} is also non-linear and non-convex and thus we cannot guarantee the uniqueness of an optimal solution.
A source of non-convexity, in particular, is the reciprocal collision avoidance constraint~\cref{eq:obs_avoid}. For example, player $i$ can potentially overtake player $j$ by passing on the left or right side and, in some situations, these two solutions might even result in an equivalent payoff.

Because of the aforementioned non-convexity, local optimization strategies will, in general, result in suboptimal solutions. In this work, however, we must calculate solutions to~\cref{eq:game_final} in a very limited amount of time for online control. We then opt for a local optimization strategy thus potentially sacrificing optimality for the sake of increasing performance.

Assume that player $i$ is at the $l$-th iteration of the Nash equilibrium search. The predicted strategy for player $j$ is then $\strategy_j^l$ and it remains fixed while player $i$ is solving problem~\cref{eq:game_final}, again, iteratively. In order to simplify the notation, in this section we drop the superscript $l$ that indicates the Nash equilibrium search iteration and, instead, we use the superscript to indicate the \emph{internal} iterations used to solve~\cref{eq:game_final}. Moreover, to clarify the notation even further, we use a $\const{\cdot}$ accent to indicate all quantities that remain constant across all inner iterations used to solve a single instance of~\cref{eq:game_final}. Therefore, assume that player $i$'s current guess of its optimal strategy is $\strategy_i^m$. We use $\strategy_{i}^{m}$ to compute a convex Quadratically-Constrained Linear (QCLP) approximation of problem~\cref{eq:game_final}.

Constraints~\cref{eq:cons_dyn,eq:cons_max_vel} can be used as they are because they are either linear or quadratic and convex. The linear approximation of~\cref{eq:obs_avoid,eq:cons_track} is also straightforward and results in the following constraints
\begin{align*}
&{\bearing_{ij}^{k,m}}^T (\const{\pos}_j^k-\pos_i^k) \ge \distMax_i \\
&\abs{{\trackNorm_i^{k,m}}^T (\pos_i^k-\track_i^{k,m})} \le \trackWidth, 
\end{align*}
with $\bearing_{ij}^{k,m}=\frac{\const{\pos}_j^k-\pos_i^{k,m}}{\norm{\const{\pos}_j^k-\pos_i^{k,m}}}$, $\trackNorm_i^{k,m}=\trackNorm(\pos_i^{k,m})$, and $\track_i^{k,m}=\track(\pos_i^{k,m})$.

The only term that requires some attention is the linear approximation of the cost function in~\cref{eq:final_iterative} and, in particular, of its first term because we do not have a closed form expression for $\trackParam_i$ as a function of $\pos_i^N$. However, since $\pos_i^N$ is a constant parameter in the optimization problem that defines $\trackParam_i$, we can exploit sensitivity analysis again to compute the derivative of $\trackParam_i$ with respect to $\pos_i^N$. To this end, let us rewrite
\begin{equation*}
\trackParam_i = \arg\min_\trackParam{d(\trackParam,\pos_i^N)},\text{with } d(\trackParam,\pos_i^N) = \tfrac{1}{2}{\norm{\track(s)-\pos_i^N}}^2.
\end{equation*}
Then, as shown in~\cite{Fiacco83} (and summarized in \cref{sec:app_track} for the case at hand) the derivative of $\trackParam_i$ with respect to $\pos_i^N$ can be calculated as
\begin{equation}\label{eq:sens_track}
\totDeriv{\trackParam_i}{\pos_i^N} = -{\left(\parDeriv[2]{d}{s}\right)}^{-1}\frac{\partial^2 d}{\partial s \partial \pos_i^N}  = \frac{\track'}{\norm{\track'}^2-{\left(\pos_i^N-\track\right)}^T \track''}.
\end{equation}
Exploiting the arc length parameterization and the relations~\cref{eq:tangent,eq:normal} we conclude
\begin{equation*}
\totDeriv{\trackParam_i}{\pos_i^N} = \frac{{\trackTan}^T}{1-\trackCurv{\left(\pos_i^{N}-\track\right)}^T \trackNorm} \defas \vect{\sigma}(\pos_i^N)
\end{equation*}
where $\trackTan,\trackNorm$ and $\track$ must be computed for $\trackParam = \trackParam_i(\pos_i^N)$.
Neglecting any term that does not depend on $\strategy_i$, the cost function can then be approximated around $\strategy_{i}^{m}$ as
\begin{equation*}
\vect{\sigma}_i^{m}\pos_{i}^{N} + \multGain \sum_{k=1}^{N} \const{\lagMult}_j^{k}{\const{\bearing}_{ij}^{k}}^T \pos_i^{k} 
\end{equation*}
with $\vect{\sigma}_i^{m}=\vect{\sigma}(\pos_i^{N,m})$.

The solution $\strategy_i^{m+1}$ to the approximate QCLP problem can then used to build a new approximation of problem~\cref{eq:game_final}. The sequential QCLP optimization terminates when either a maximum number of iterations has been reached or the difference between two consecutive solutions, $\residual = \norm{\strategy_i^{m+1}-\strategy_i^{m}}$, is smaller than a given threshold.

\subsection{Alternative control strategies}\label{sec:alternative_ctrl}
In order to asses the effectiveness of our approach, in the experiments of \cref{sec:experiments}, we let our controller compete with the following alternative control strategies.

\subsubsection{Model predictive control (MPC)}
This strategy is based on the realistic, but na\"ive, assumption that player $i$'s opponent will follow a straight line trajectory at (constant) maximum linear velocity along the local direction of the track, i.e.~$\linVelMax_j\trackTan(\trackParam(\pos_j^0))$. Based on this assumption, player $i$ can predict player $j$'s strategy and solve~\cref{eq:game_discrete} as a single classical optimal control problem. The numerical optimization scheme described in \cref{sec:optim_ctrl} can be used also in this case to efficiently compute a locally optimal solution.

\subsubsection{Reciprocal velocity obstacles (RVO)}
This second benchmark strategy is based on the multi-agent collision avoidance library proposed in~\cite{vanGuyLinMan2011} and implemented in the open-source library RVO2\footnote{\url{http://gamma.cs.unc.edu/RVO2/}}. We approximated the boundaries of the track with a set of polygonal obstacles. For each robot, RVO uses a reference linear velocity with maximum norm $\linVelMax_i$ and direction computed as $\trackTan(\pos_i^0)+\rho(\track(\trackParam_i^0)-\pos_i^0)$ where $\rho>0$ is a free parameter that allows to trade off between the first term, which makes the robot follow the local direction of the track, and the second one, which keeps the robot close to the center line.


\subsection{Orientation control}\label{sec:ori_ctrl}

Since the robots' onboard cameras have a limited field of view (see~\cref{eq:fov_constr}), each robot needs to actively maintain visibility of its opponent, despite both players motion, by exploiting the yaw degree of freedom.

A simple, yet effective, strategy to maintain visibility is to always align the camera axis to the relative bearing vector $\bearing_{ij}$ thus maintaining the opponent in the center of the image.

Given the planned (respectively predicted) trajectory for player $i$ and its opponent, we can then calculate the desired yaw angle for player $i$ as
\begin{equation*}
  \yaw_i^k = \atan(\ey^T\bearing_{ij}^k,\ex^T\bearing_{ij}^k),\text{with } \bearing_{ij}^k = \frac{\pos_j^k-\pos_i^k}{\norm{\pos_j^k-\pos_i^k}}.
\end{equation*}

A desired angular velocity can also be computed by differentiating consecutive samples as
$\yawRate_i^k = \frac{\yaw_i^k - \yaw_i^{k-1} }{\delta t}$.

\section{OPPONENT POSITION ESTIMATION}\label{sec:estimation}

The proposed planning strategy requires each agent to know both players positions ($\pos_1^0,\pos_2^0$) at the beginning of each planning phase.
As already mentioned, we assume that each robot knows its own position from onboard sensors. On the other hand, because of the lack of communication between the players, the position of the opponent must be estimated by fusing the visual and inertial measurements from onboard camera and IMU sensors.
In our implementation, we first exploit the onboard camera and gyro, to estimate the opponent position expressed in the local body frame of robot $i$, i.e.~$\pos_{ij}$. 
Then, we transform the final estimate into the world reference frame using the available ego state estimates.

The belief over the opponent's relative state is maintained via a Kalman Filter and the expected value of this belief is used as the opponent's state estimate in the final solution to Problem~\cref{eq:game_final}.

In order to be robust with respect to altitude control errors and robot roll and pitch rotations, for estimation purposes, we consider a 3D dynamical model.
We approximate the relative dynamics of opponent $j$ with respect to $i$ as a second order kinematic model. Assuming constant world-frame linear velocities for both robots (i.e.~$\dot\linVel_{i}=\dot\linVel_{j}=0$), differentiating~\cref{eq:bearing} we obtain
\begin{equation}
\label{eq:kf_dynamics}
\begin{bmatrix} \dot{\pos}_{ij}  \\ \dot{\linVel}_{ij} \end{bmatrix} =
\begin{bmatrix} 	-\skewZ(\angVel_i) 	& \eye{3} \\
				\mathbf{0}_{3\times 3} 		& -\skewZ(\angVel_i)  \end{bmatrix}
\begin{bmatrix} \pos_{ij} \\ \linVel_{ij} \end{bmatrix} +
\vect{w}
\, .
\end{equation}
In these dynamics, $\linVel_{ij}=\rotMat{}{i}^T(\linVel_j-\linVel_i)$, $\skewZ(\angVel_i)$ is the skew symmetric matrix built with the components of robot $i$'s body frame rotation rates $\angVel_i$ ---measured via gyroscope---and $\vect{w} \sim \mathcal{N}(\mathbf{0}_{6},\mathbf{Q})$ is additive, zero-mean Gaussian white noise with covariance matrix $\mathbf{Q} \in \real{6\times 6}$.

As discussed below, robot $i$ can measure the opponent's relative position using an onboard camera, i.e.
\begin{equation}
\label{eq:KF_measurement}
	\vect{y}_i = \pos_{ij} + \vect{v}
	\, ,
\end{equation}
where $\vect{y}_i \in \real{3}$, $\vect{v} \sim \mathcal{N}(\mathbf{0}_{3},\mathbf{R})$ is additive, zero-mean Gaussian process noise with covariance matrix $\mathbf{R} \in \real{3\times 3}$.
\Cref{eq:kf_dynamics,eq:KF_measurement} form a time-varying linear system with additive Gaussian input and measurement noise. Standard Kalman filtering techniques can then be applied to design an estimator.


We now detail the image processing pipeline that allows to retrieve $\vect{y}_i$ from the onboard camera images.
As mentioned in \cref{sec:formulation}, each robot competing in the race is fitted with a colored sphere of known (by both robots) color and radius $\const{r}_s$. The sphere center lies (approximately) on the vertical axis of the robot so that its position in the world frame is not affected by yaw rotations. The distance between the sphere and the robot frame is also known.
Finally, we assume that the extrinsic and intrinsic camera parameters are calibrated for each robot, the latter of which are modeled using the pinhole camera model~\cite{ma2012invitation}.


As discussed in~\cite{FoCh09,spica2014active}, the projection of the opponent's sphere on a player's camera image plane is, in general, an ellipse.
If this latter can be segmented from the image, then the 3D position of the sphere center with respect to the camera can be directly expressed in terms of $\const{r}_s$ and a set of image moments, up to the 2nd order, measured from the segmented area.

In our implementation, the robots extract the colored blob in the image frame via thresholding with the known color. The most circular blob that is within a tracking window from the previous blob's location is used. The image moments of the extracted blob are then calculated with OpenCV\footnote{\url{https://opencv.org/}} image moment functions.

Using the known robot-to-sphere transformation and assuming that the opponent is hovering, we can finally compute the 3D position of the opponent, i.e.~the filter measurement in~\cref{eq:KF_measurement}.

\section{RESULTS}\label{sec:experiments}

\subsection{Simulations}

In order to validate our approach, we performed an extensive simulation campaign.
We used the open-source RotorS package~\cite{rotors:2016} to simulate the full quadrotor dynamics.
To avoid rendering the onboard camera images and speed up the simulations, we did not make use of the vision-based tracking and estimation algorithms described in \cref{sec:estimation}. Instead, the two robots have access to each other's position, provided by RotorS.
Our planning algorithm was implemented in C++ and interfaced with the simulator using ROS.
We used a simulation time step of \SI{10}{\milli\second}, but we run our planners at \SI{20}{\hertz}.
We also used state-of-the-art nonlinear controllers to drive our quadrotors along the optimal trajectory resulting from the solution of~\cref{eq:game_final}.
We refer the reader to~\cite{Faessler15icra} and~\cite{Faessler16jfr} for further information about the control pipeline.

The two simulated robots have a radius of \SI{.3}{\meter} and maintain a minimum relative distance $\distMax$ of \SI{.8}{\meter} from their opponent.
The simulated track is represented in \cref{fig:track}. The track fits into a \SI{15}{\meter} $\times$ \SI{11}{\meter} rectangle and its half-width $\trackWidth$ is \SI{1.5}{\meter}. The origin of the world frame was set at the center of the longest straight segment of the track.

We let our game-theoretic planner and the alternative strategies described in \cref{sec:alternative_ctrl} compete against each other over multiple races differing by the robots initial positions. In order to enforce some interaction, we set the maximum linear velocities of the two robots to \SI{.5}{\meter/\second} and \SI{.6}{\meter/\second} and we made the faster robot always start behind the slower one. In particular, for each race, we sampled the initial position of the faster robot from a uniform distribution in the rectangle $[-0.1,1.5]\times[-0.7,0.7]$. Similarly, the position of the slower robot was sampled from the rectangle $[1.6,1.7]\times[-0.7,0.7]$. We discarded any pair of sampled initial positions that would violate the collision avoidance constraints~\cref{eq:coll_avoid}.

We ran a total of 900 simulations in which the same 150 sampled initial conditions were used for each of the following scenarios:
\begin{expcases}
\item fast GTP vs. slow MPC; \label{case:gtp_mpc}
\item fast MPC vs. slow GTP; \label{case:mpc_gtp}
\item fast GTP vs. slow RVO; \label{case:gtp_rvo}
\item fast RVO vs. slow GTP; \label{case:rvo_gtp}
\item fast MPC vs. slow RVO; \label{case:mpc_rvo}
\item fast RVO vs. slow MPC. \label{case:rvo_mpc}
\end{expcases}
Here, and in the rest of this section, the acronym GTP indicates the Game Theoretic Planner developed in this paper.
We terminated each simulation as soon as one robot completed an entire track loop and reached the finish line positioned at $x=$~\SI{2.32}{\meter}.

\begin{figure}[tb]
\centering
\begin{subfigure}{0.5\columnwidth}
\centering
\includegraphics[width=\columnwidth]{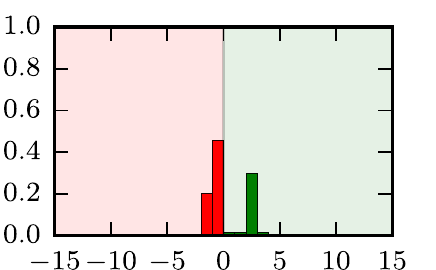}
\caption{\Cref{case:gtp_mpc}: fast GTP vs. slow MPC}\label{fig:histo_1}
\end{subfigure}%
\begin{subfigure}{0.5\columnwidth}
\centering
\includegraphics[width=\columnwidth]{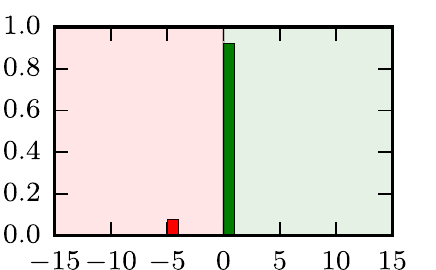}
\caption{\Cref{case:mpc_gtp}: fast MPC vs. slow GTP}\label{fig:histo_2}
\end{subfigure}
\begin{subfigure}{0.5\columnwidth}
\centering
\includegraphics[width=\columnwidth]{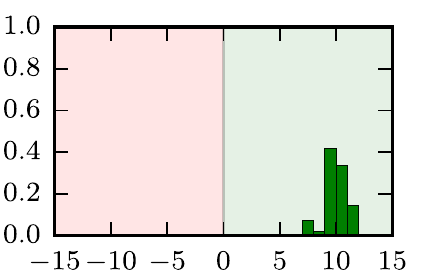}
\caption{\Cref{case:gtp_rvo}: fast GTP vs. slow RVO}\label{fig:histo_3}
\end{subfigure}%
\begin{subfigure}{0.5\columnwidth}
\centering
\includegraphics[width=\columnwidth]{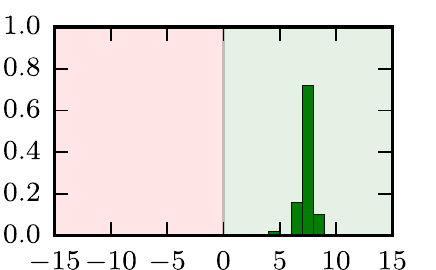}
\caption{\Cref{case:rvo_gtp}: fast RVO vs. slow GTP}\label{fig:histo_4}
\end{subfigure}
\begin{subfigure}{0.5\columnwidth}
\centering
\includegraphics[width=\columnwidth]{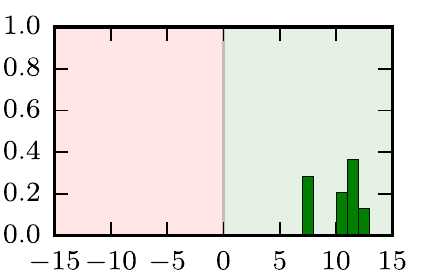}
\caption{\Cref{case:mpc_rvo}: fast MPC vs. slow RVO}\label{fig:histo_5}
\end{subfigure}%
\begin{subfigure}{0.5\columnwidth}
\centering
\includegraphics[width=\columnwidth]{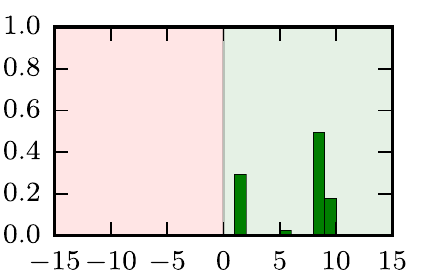}
\caption{\Cref{case:rvo_mpc}: fast RVO vs. slow MPC}\label{fig:histo_6}
\end{subfigure}%
\caption{Histogram representation of the final arch-length difference (as in~\cref{eq:objective}) between the two robots for all of the simulations. The simulations are divided by case as indicated in the captions. In~(\protect\subref{fig:histo_1}) to~(\protect\subref{fig:histo_4}), the distance is calculated in such a way that it is positive when the robot controlled using GTP wins the race and negative otherwise. In~(\protect\subref{fig:histo_5}) and~(\protect\subref{fig:histo_6}), instead, a positive distance indicates a victory for the MPC planner over RVO. A green and red coloring is also used to highlight positive and negative parts of the histogram.}
\label{fig:sim_histograms}
\end{figure}

\begin{figure}[tb]
\centering
\begin{subfigure}{0.5\columnwidth}
\centering
\includegraphics[width=\columnwidth]{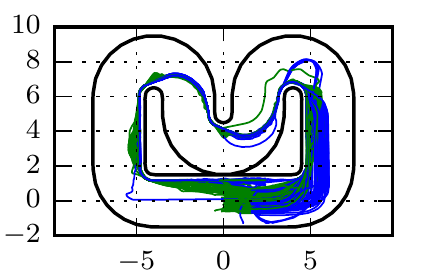}
\caption{\Cref{case:gtp_mpc}: fast GTP vs. slow MPC}\label{fig:trace_1}
\end{subfigure}%
\begin{subfigure}{0.5\columnwidth}
\centering
\includegraphics[width=\columnwidth]{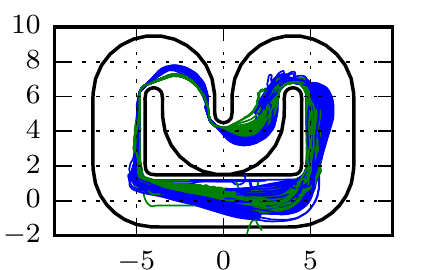}
\caption{\Cref{case:mpc_gtp}: fast MPC vs. slow GTP}\label{fig:trace_2}
\end{subfigure}
\begin{subfigure}{0.5\columnwidth}
\centering
\includegraphics[width=\columnwidth]{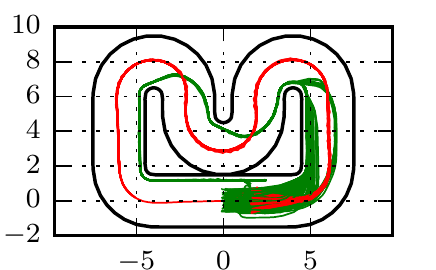}
\caption{\Cref{case:gtp_rvo}: fast GTP vs. slow RVO}\label{fig:trace_3}
\end{subfigure}%
\begin{subfigure}{0.5\columnwidth}
\centering
\includegraphics[width=\columnwidth]{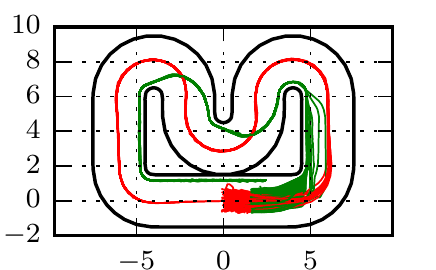}
\caption{\Cref{case:rvo_gtp}: fast RVO vs. slow GTP}\label{fig:trace_4}
\end{subfigure}
\begin{subfigure}{0.5\columnwidth}
\centering
\includegraphics[width=\columnwidth]{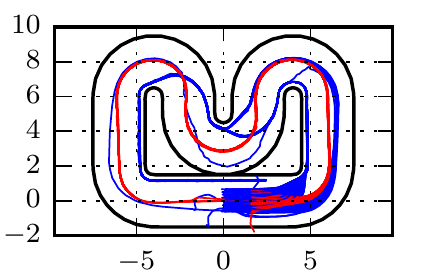}
\caption{\Cref{case:mpc_rvo}: fast MPC vs. slow RVO}\label{fig:trace_5}
\end{subfigure}%
\begin{subfigure}{0.5\columnwidth}
\centering
\includegraphics[width=\columnwidth]{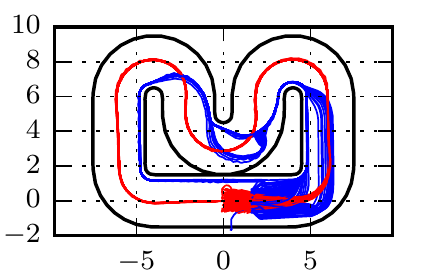}
\caption{\Cref{case:rvo_mpc}: fast RVO vs. slow MPC}\label{fig:trace_6}
\end{subfigure}%
\caption{Position traces for the two competing robots for all of the simulations. The simulations are divided by case as indicated in the captions and the following color code was used for the planners: GTP -- green, MPC -- blue, RVO -- red.}
\label{fig:sim_traces}
\end{figure}

In \cref{fig:sim_histograms}, we report an histogram representation of the final distance along the track (i.e.~the arch-length difference~\cref{eq:objective}) between the two robots. In \cref{case:gtp_mpc,case:mpc_gtp,case:gtp_rvo,case:rvo_gtp}, the distance is calculated in such a way that it is positive when the robot controlled using GTP wins the race and negative otherwise. In \cref{case:mpc_rvo,case:rvo_mpc}, instead, a positive distance indicates a victory for the MPC planner over RVO.
A green and red coloring is also used to highlight positive and negative parts of the histogram. 

In \cref{fig:sim_traces}, we report the position traces for the two competing robots for all of the simulations. The traces are divided by case and we used the following color code: GTP -- green, MPC -- blue, RVO -- red.

First of all, as it can be seen from \cref{fig:histo_3,fig:histo_4,fig:histo_5,fig:histo_6}, the RVO strategy is clearly the least effective one among the three alternatives considered in this paper. Regardless of the initial placing and of the possible advantage in terms of maximum speed, the robots controlled with this strategy lost all races. We believe that the main reason for this poor performance is the fact that RVO is a reactive (instantaneous) control strategy while both GTP and MPC use an extended planning horizon. Because of the lack of planning, RVO does not \emph{anticipate} (and appropriately cut) the curves and ends up following a longer trajectory. This can clearly be noticed by looking at \cref{fig:trace_3,fig:trace_4,fig:trace_5,fig:trace_6}. The performance could possibly be improved by considering alternative heuristics in the calculation of the reference/desired velocity for the RVO algorithm.

The comparison between GTP and MPC is somewhat more fair because both strategies effectively follow the track and the only difference between the two lies in the way they interact with the opponent (the two algorithms are perfectly identical when the two robots do not interact). A direct comparison between the two strategies is provided by
\cref{fig:histo_1,fig:histo_2}. From \cref{fig:histo_1}, we can notice that, when the drone running the GTP planner is faster than the MPC one, it manages to overtake the MPC planner, which starts from an advantageous position, in approximately 30\% of the races. A closer look at the simulations reveals that, in this scenario, the GTP planner often tends to "overestimate" its opponent (it assumes the opponent is using GTP as well). The consequence of this is that, when attempting to overtake, it expects the opponent to block its motion and ends up following an overly cautions trajectory moving sideways along the track more than necessary.

On the other hand, when the GTP is slower, it manages to defend its initial advantage for the vast majority of the races (see \cref{fig:histo_2}). Looking at \cref{fig:trace_2}, we can clearly visualize the strategy adopted by the GTP planner to defend its position, especially towards the end of the race (the bottom straight part of the track). The GTP planner clearly moves sideways along the track to block the MPC planner thus exploiting the collision avoidance constraint to its own advantage. The MPC planner, instead, cannot adopt a similar strategy because it does not properly model the reactions of its opponent. On the contrary, by assuming that the opponent will move straight along the path, completely careless of possible collisions, the MPC planner is often forced to make room to the opponent because of its own collision avoidance constraints (see, for example, the blue traces in the top right part of \cref{fig:trace_1}).

An indirect comparison between GTP and MPC can also be done by analyzing how they both perform against RVO in similar situations. Both when competing against a slower robot (see \cref{fig:histo_3,fig:histo_5}) and against a faster one (see \cref{fig:histo_4,fig:histo_6}) the GTP planner tends to win with a slightly larger separation in average and a much more narrow distribution of final distances. The improvement is more significant when the GTP is playing in a "defensive" role, i.e.~it is controlling a slower robot with an initial advantage. We believe that this is due, once again, to an overly cautious behavior when attempting to overtake the opponent in \cref{case:gtp_mpc,case:gtp_rvo}.

\subsection{Experiments}

We also validated our approach by implementing and testing it on real hardware.

Our quadrotors are based on the DJI F330 frame and equipped with both off-the-shelf and custom-made components.
The robots have a take-off weight of approximately \SI{900}{\gram} and a diagonal rotor distance of \SI{33}{\centi\meter}.

A PX4FMU autopilot provides, among other sensors, an Inertial Measurement Unit and a micro-controller running a low-level controller for the robot orientation and bodyrates.
An Odroid single-board computer is used to run the algorithm described in \cref{sec:estimation} in order to estimate the state of the other vehicle from images acquired by a forward-facing MatrixVision mvBlueFOX-MLC200w $752\!\times\!480$-pixel RGB camera with a framerate of \SI{15}{fps}.
The motion planner and the high-level position controller, instead, were run on a ground-station which communicates with the onboard Odroids through Wi-Fi.
Actuation is provided by four T-MOTOR MT2208 motors, controlled by Dys XSD 20A speed controllers.

In addition to the onboard sensors, each robot is also provided, through wireless communication, with ego pose measurements from an Optitrack motion-capture system. This information is fused onboard with inertial measurements in an Extended Kalman Filter~\cite{Lynen13iros} to estimate the full robot state.

In our experiment, both robots were running the game-theoretic planner described in this paper. They both had a maximum speed limitation of \SI{0.6}{\meter/\second} and assumed \SI{0.3}{\meter/\second} limitation for their opponent. The minimum allowed distance between the two was \SI{0.6}{\meter}.
The experimental track center line was a square of \SI{2.3}{\meter} sides with rounded corners and a half-width $\trackWidth$ of \SI{0.9}{\meter}.

Some snapshots of the experiment are reported in \cref{fig:exp_snaps} but we highly encourage the reader to visualize the attached video (also available at \videourl{}) to better appreciate the robots behavior.
In the snapshots (as well as in the video) we used solid red and blue lines to represent the trajectory currently planned by each robot and dashed lines of opposite colors for the trajectory predicted, for the same robot, by the adversary.

\begin{figure*}[b]
\centering
\includegraphics[width=\columnwidth]{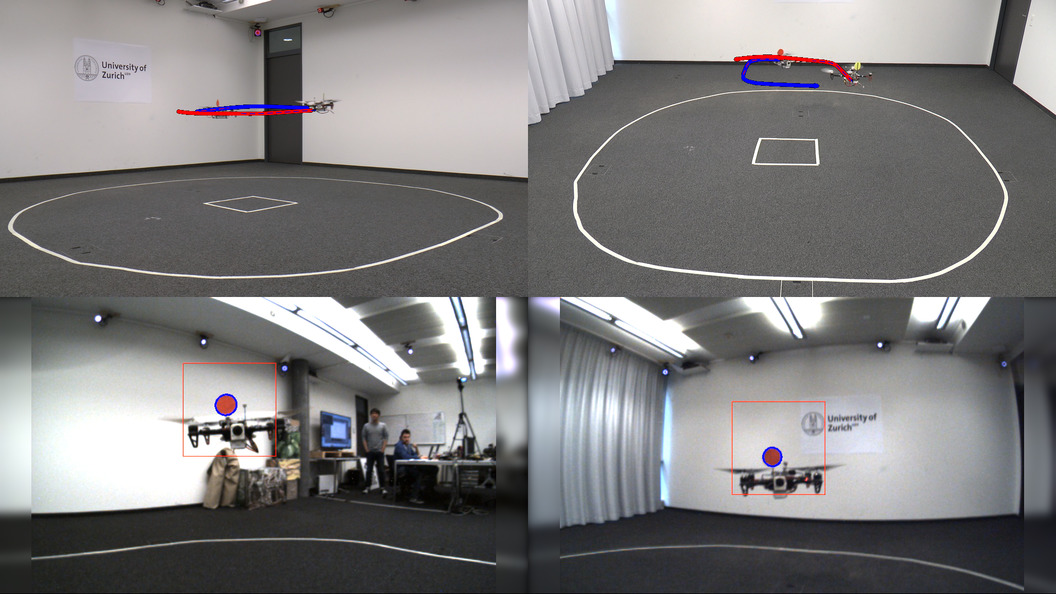}
\includegraphics[width=\columnwidth]{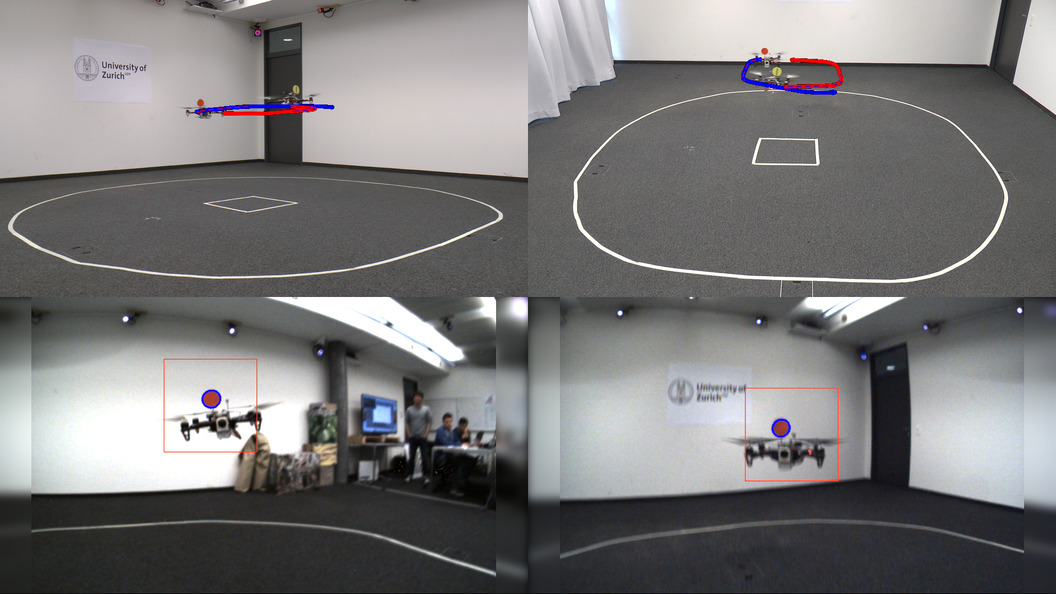}\vspace{1.5mm}
\includegraphics[width=\columnwidth]{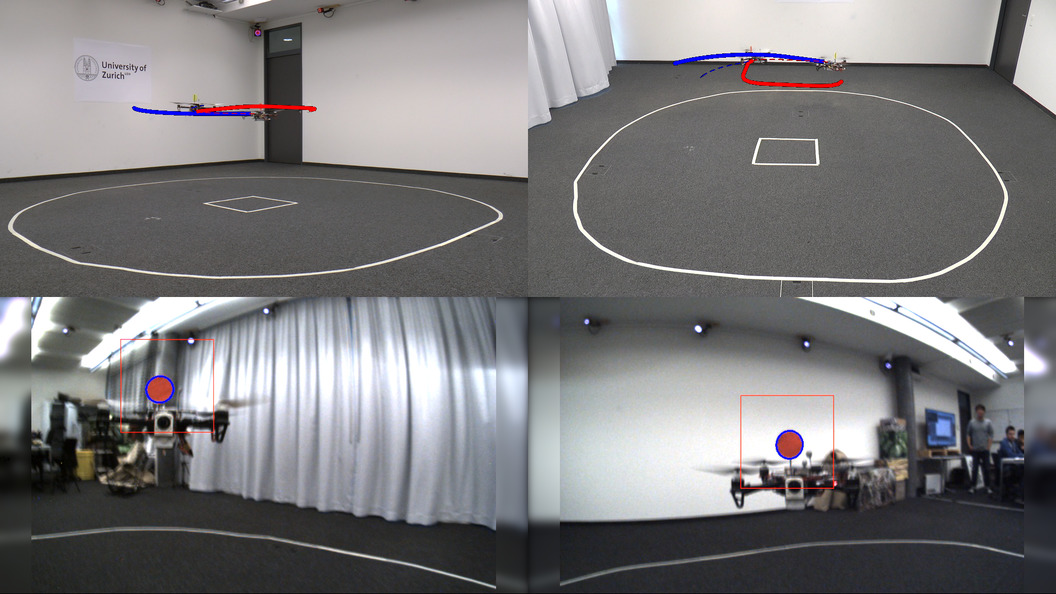}
\includegraphics[width=\columnwidth]{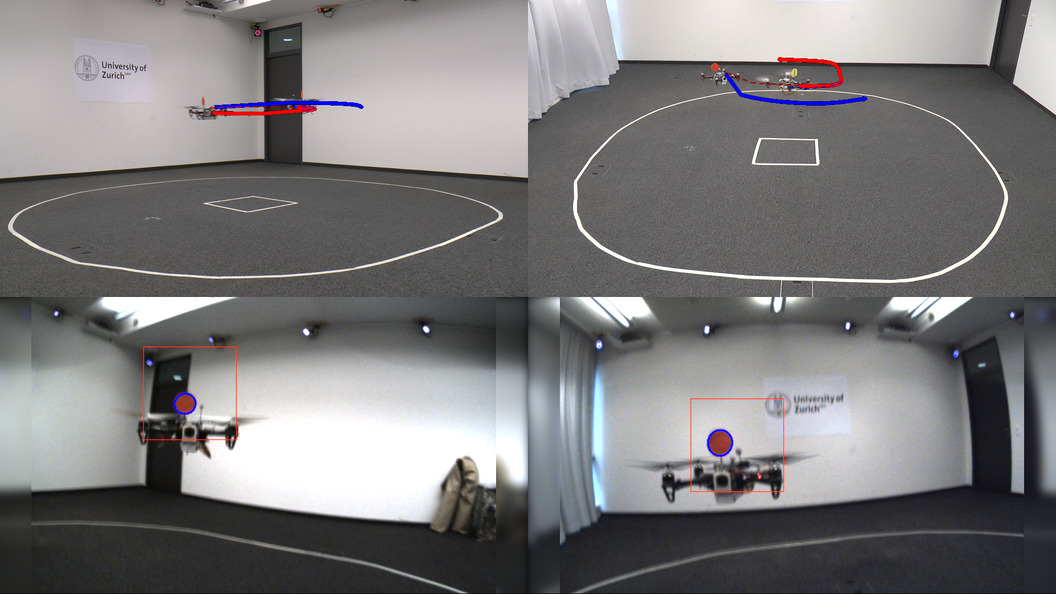}
\caption{Snapshots of the real hardware experiments using our game-theoretic planner for two-player drone racing. The top row shows two external views of the experiment. Solid thick lines indicate the current trajectory planned by each robot for itself. Dashed thinner line of the same color indicate the predicted opponent trajectory. The bottom row shows images from the robots' onboard cameras with an overlay of the visual tracking results. We highly encourage the reader to view the attached video, which is also available at \videourl{}.}
\label{fig:exp_snaps}
\end{figure*}

As it can be noticed by the video, our experimental space is not sufficiently large for the robots to overtake each other given their size and the track width. However, one can still appreciate how each robot continuously replans its own and its opponent trajectory as well as the interactions that each player predicts for the future. In particular, one can notice that the red robot mostly follows the track without expecting to have any interaction with its opponent (which it assumes to be slower). On the contrary, the blue robot plans to overtake the adversary and, in some cases, it expects this latter to block it by moving sideways (see, for example, the blue lines in the third snapshot).
The video also demonstrates, in the bottom, the performance of our vision-based tracking and estimation algorithms discussed in \cref{sec:estimation}.

\section{CONCLUSIONS}\label{sec:conclusion}

In this paper we described a novel online motion planning algorithm for two-player drone racing.
By exploiting sensitivity analysis within an iterated best response algorithm, our planner can effectively model and even exploit the opponent's reactions.

From a theoretical point of view, we showed that, if the iterative resolution strategy converges to a solution, then this latter satisfies necessary conditions for a Nash equilibrium.
Moreover, we demonstrated the effectiveness of our approach through an extensive set of simulations in which our planner was let to compete against two alternative, and well established, approaches.
Finally, we also presented a vision-based tracking and estimation algorithm that can efficiently estimate the opponent's pose.

Both our planner and our estimation strategies can run in real time and their performance was demonstrated through experimental tests on real hardware.

Despite the encouraging results, our planner still presents some weakness. First of all, because of the non-convexity of our problem, the optimization algorithm can converge to local minima. In the future, we want to investigate the use of mixed integer approaches to better handle the non-convex constraints of our problem.
In addition to this, our planner assumes that the opponent is using a similar planning strategy with a known cost function. On the one hand, our algorithm showed very good performance when competing against alternative strategies (which clearly violate this assumption). On the other hand, we noticed that, in certain situations, our planner can "overestimate" the intelligence of the adversary and generate overly conservative trajectories. In this respect, it would be interesting to couple our strategy with an online approach for learning the opponent policy and/or cost function.
We also plan to extend our results to races involving more than two players.

Finally, we believe that similar game theoretic approaches would prove successful in a number of applications, particularly in autonomous driving scenarios involving multiple cars and/or pedestrians.

\appendix
\subsection{Proof of \texorpdfstring{\cref{lemma:sensitivity}}{Lemma~\ref{lemma:sensitivity}}}\label{sec:app_sensitivity}

In order to simplify the notation as much as possible, in this subsection we consider a streamlined form for the optimization problem of the form
\begin{equation}\label{eq:simple_optim}
\begin{aligned}
\max_x~&\objFinal(x) \\
\text{s.t. } &\inEqConJointEl(x,c) = 0
\end{aligned}
\end{equation}
where $c$ is a scalar parameter and $\objFinal$ and $\inEqConJointEl$ are scalar differentiable functions of their arguments.
For each value of $c$, let us indicate with $x^*(c)$ the solution of~\cref{eq:simple_optim} and with $\bestresponse(c) = \objFinal(x^*(c))$ the associated optimal outcome.
We want to study how the optimal cost $\bestresponse$ changes when $c$ changes around a point $\const{c}$, i.e.
\begin{equation}\label{eq:sens_deriv}
  \evalin{\totDeriv{\bestresponse(c)}{c}}{\const{c}} = \evalin{\totDeriv{\objFinal(x^*(c))}{c}}{\const{c}} = 
   \evalin{\totDeriv{\objFinal(x)}{x}}{x^*(\const{c})}\evalin{\totDeriv{x^*(c)}{c}}{\const{c}}
\end{equation}
Since, for all $c$, $x^*(c)$ is an optimal solution to~\cref{eq:simple_optim}, it must satisfy the KKT necessary optimality conditions associated to~\cref{eq:simple_optim}, i.e.
\begin{align}
  &\evalin{\totDeriv{\objFinal(x)}{x}}{x^*} - \lagMult \evalin{\parDeriv{\inEqConJointEl(x,c)}{x}}{x^*} = 0 \label{eq:con_at_opt_1}\\
  &\inEqConJointEl(x^*(c),c) = \inEqConJointEl^*(c) = 0 \label{eq:con_at_opt_2}
\end{align}
where $\lagMult$ is the Lagrange multiplier associated to the equality constraint.
Isolating the first term in~\cref{eq:con_at_opt_1} and substituting it in~\cref{eq:sens_deriv} we obtain
\begin{equation}\label{eq:sens_deriv1}
  \evalin{\totDeriv{\bestresponse(c)}{c}}{\const{c}} = 
  \lagMult \evalin{\parDeriv{\inEqConJointEl(x,c)}{x}}{x^*(\const{c})}\evalin{\totDeriv{x^*(c)}{c}}{\const{c}}.
\end{equation}
Note that, since~\cref{eq:con_at_opt_2} must remain true for all $c$, its total derivative w.r.t.~$c$ must also be zero, i.e.
\begin{equation}\label{eq:con_at_opt_deriv}
  \totDeriv{\inEqConJointEl^*(c)}{c} = \evalin{\parDeriv{\inEqConJointEl(x,c)}{x}}{x^*}\totDeriv{x^*(c)}{c}  + \evalin{\parDeriv{\inEqConJointEl(x,c)}{c}}{x^*} = 0.
\end{equation}
Isolating the first term from~\cref{eq:con_at_opt_deriv} and substituting it in~\cref{eq:sens_deriv1}, we finally conclude that
\begin{equation*}
  \evalin{\totDeriv{\bestresponse(c)}{c}}{\const{c}} = 
   - \lagMult \evalin{\parDeriv{\inEqConJointEl(x,c)}{c}}{x^*(\const{c})},
\end{equation*}
which reduces to~\cref{eq:sensitivity} for $x=\strategy_j$, $\const{c}=\strategy_i^{l-1}$, and $x^*=\strategy_j^l$.

This proof can trivially be extended to problems with multiple joint constraints or with additional constraints that do not depend on $c$ (their derivatives with respect to $c$ will simply be null).
If the problem contains inequality constraints, instead, under the assumption that, in the vicinity of $c$, the set of active constraints remains the same, the proof can readily be applied by just considering an equivalent problem in which any active inequality constraint is transformed into an equality constraints and any inactive constraint is ignored.

\subsection{Proof of \texorpdfstring{\cref{thm:nash}}{Theorem~\ref{thm:nash}}}\label{sec:app_nash}
Applying Karush-Kuhn-Tucker conditions to~\cref{eq:pli_optim_prob} one obtains the following set of necessary conditions for a Nash equilibrium $(\strategy_1^*,\strategy_2^*)$ and the associated Lagrange multipliers
\begin{subequations}
\begin{align}
&\parDeriv{\objFinal_i}{\strategy_i}(\strategy_i^*) - \lagMultInEq_i^* \parDeriv{\inEqConJoint_i}{\strategy_i}(\strategy_i^*,\strategy_j^*) \label{eq:kkt_game1} \\
&\qquad -\lagMultEqOwn_i^* \parDeriv{\eqConOwn_i}{\strategy_i}(\strategy_i^*) - \lagMultInEqOwn_i^* \parDeriv{\inEqConOwn_i}{\strategy_i}(\strategy_i^*) = 0 \nonumber\\
&\eqConOwn_i(\strategy_i^*) = 0 \\
&\inEqConOwn_i(\strategy_i^*) \le 0 \\
&\lagMultInEqOwn_i^* \inEqConOwn_i(\strategy_i^*) = 0, \lagMultInEqOwn_i^*\ge 0 \\
&\inEqConJoint_i(\strategy_i^*,\strategy_j^*) \le 0 \label{eq:kkt_game2} \\
&\lagMultInEq_i^* \inEqConJoint_i(\strategy_i^*,\strategy_j^*) = 0, \lagMultInEq_i^*\ge 0\label{eq:kkt_game3}
\end{align}
\end{subequations}

Now assume that the iterative algorithm described in \cref{sec:game} converges to a solution $(\strategy_1^l,\strategy_2^l)$, i.e.~$\strategy_i^{l+1} = \strategy_i^{l}$ for both players. Then, by applying the KKT conditions to problem~\cref{eq:final_iterative},  $(\strategy_1^l,\strategy_2^l)$ must satisfy
\begin{subequations}
\begin{align}
&\parDeriv{\objFinal_i}{\strategy_i}(\strategy_i^l) + \multGain_i\lagMultInEq_j^l \parDeriv{\inEqConJoint_j}{\strategy_i}(\strategy_i^l,\strategy_j^l) -\lagMultInEq_i^l \parDeriv{\inEqConJoint_i}{\strategy_i}(\strategy_i^l,\strategy_j^l) \\
&\qquad\lagMultEqOwn_i^l \parDeriv{\eqConOwn_i}{\strategy_i}(\strategy_i^l) - \lagMultInEqOwn_i^l \parDeriv{\inEqConOwn_i}{\strategy_i}(\strategy_i^l) = 0 \nonumber\\
&\eqConOwn_i(\strategy_i^l) = 0 \\
&\inEqConOwn_i(\strategy_i^l) \le 0 \\
&\lagMultInEqOwn_i^l \inEqConOwn_i(\strategy_i^l) = 0, \lagMultInEqOwn_i^l\ge 0 \\
&\inEqConJoint_i(\strategy_i^l,\strategy_j^l) \le 0 \\
&\lagMultInEq_i^l \inEqConJoint_i(\strategy_i^l,\strategy_j^l) = 0, \lagMultInEq_i^l\ge 0 \label{eq:kkt_app}
\end{align}
\end{subequations}

If one additionally has 
$\parDeriv{\inEqConJoint_i}{\strategy_i}(\strategy_i^l,\strategy_j^l) = \parDeriv{\inEqConJoint_j}{\strategy_i}(\strategy_i^l,\strategy_j^l)$ (as it is the case for our problem), then one can see that $(\strategy_1^l,\strategy_2^l)$ satisfy \crefrange{eq:kkt_game1}{eq:kkt_game2} with $\lagMultEqOwn_{i}^*=\lagMultEqOwn_{i}^l, \lagMultInEqOwn_{i}^*=\lagMultInEqOwn_{i}^l$ and
$\lagMultInEq_i^*=\lagMultInEq_i^l-\multGain_i\lagMultInEq_j^l$.
In order to satisfy~\cref{eq:kkt_game3}, however, one also needs to impose that:
\begin{subequations}
\begin{align}
&(\lagMultInEq_1^l-\multGain_1\lagMultInEq_2^l)\inEqConJoint_1(\strategy_1^l,\strategy_2^l) = 0 \label{eq:kkt_add1}\\
&\lagMultInEq_1^l \ge \multGain_1\lagMultInEq_2^l \label{eq:kkt_add2} \\
&(\lagMultInEq_2^l-\multGain_2\lagMultInEq_1^l)\inEqConJoint_2(\strategy_1^l,\strategy_2^l) = 0 \label{eq:kkt_add3}\\
&\lagMultInEq_2^l \ge \multGain_2\lagMultInEq_1^l. \label{eq:kkt_add4}
\end{align}
\end{subequations}
Using~\cref{eq:kkt_app}, \cref{eq:kkt_add1,eq:kkt_add3} reduce to
\begin{subequations}
\begin{align*}
&\multGain_1\lagMultInEq_2^l\inEqConJoint_1(\strategy_1^l,\strategy_2^l) = 0 \\
&\multGain_2\lagMultInEq_1^l\inEqConJoint_2(\strategy_1^l,\strategy_2^l) = 0.
\end{align*}
\end{subequations}

Exploiting, again,~\cref{eq:kkt_app}, this condition is satisfied if $\inEqConJoint_1(\strategy_1^l,\strategy_2^l) = \inEqConJoint_2(\strategy_1^l,\strategy_2^l)$ for all active constraints and if the sets of active constraints are the same for both players (i.e.~$\lagMultInEq_1^l>0\iff\lagMultInEq_2^l>0$). Both these conditions are satisfied if, as it is the case for our application, $\inEqConJoint_1(\strategy_1^l,\strategy_2^l) = \inEqConJoint_2(\strategy_1^l,\strategy_2^l)$.
As for~\cref{eq:kkt_add2,eq:kkt_add4}, instead, if $\inEqConJoint_i(\strategy_i,\strategy_j) = \inEqConJoint_j(\strategy_i,\strategy_j)$, one can enforce it by making $\multGain_i$ arbitrarily small. 

\subsection{Proof of \texorpdfstring{\cref{eq:sens_track}}{(\ref{eq:sens_track})}}\label{sec:app_track}
Consider the following optimization problem:
\begin{equation*}
\min_\trackParam{d(\trackParam,\pos_i^N)}.
\end{equation*}
We can interpret $\pos_i^N$ as a constant parameter and study how the solution $\trackParam_i$ to the above problem changes when $\pos_i^N$ changes around a point $\const{\pos}_i^N$.
Under the optimality assumption, for each value $\pos_i^N$, the corresponding solution $\trackParam_i(\pos_i^N)$ must satisfy the following necessary condition
\begin{equation}\label{eq:track_par_deriv}
\evalin{\parDeriv{d(\trackParam,\pos_i^N)}{\trackParam}}{\trackParam_i(\pos_i^N)} = 0.
\end{equation}
Note that the left hand side of~\cref{eq:track_par_deriv} is a function of $\pos_i^N$ only and it must be zero for all $\pos_i^N$. Therefore, its derivative with respect to $\pos_i^N$ must also be zero
\begin{equation*}
\begin{aligned}
0 &= \totDeriv{}{\pos_i^N}\left[\evalin{\parDeriv{d(\trackParam,\pos_i^N)}{\trackParam}}{\trackParam_i(\pos_i^N)}\right] \\
&=\evalin{\parDeriv[2]{d(\trackParam,\pos_i^N)}{\trackParam}}{\trackParam_i(\pos_i^N)} \totDeriv{\trackParam_i(\pos_i^N)}{\pos_i^N} +
\MixParDeriv[2]{d(\trackParam,\pos_i^N)}{\partial \trackParam \partial \pos_i^N}.
\end{aligned}
\end{equation*}
We can, then, conclude that:
\begin{equation*}
 \totDeriv{\trackParam_i(\pos_i^N)}{\pos_i^N}= - {\left[\evalin{\parDeriv[2]{d(\trackParam,\pos_i^N)}{\trackParam}}{\trackParam_i(\pos_i^N)}\right]}^{-1}
\MixParDeriv[2]{d(\trackParam,\pos_i^N)}{\partial \trackParam \partial \pos_i^N}
\end{equation*}
q.e.d.


\IEEEtriggeratref{29}
\bibliographystyle{IEEEtran}
\bibliography{IEEEfull,uavgame}

\end{document}